%% file: sample_paper_aj8.tex
\documentclass[preprint,11pt]{article}

\usepackage{fullpage}
\usepackage{amssymb,amsfonts,amsmath,amsthm,amscd,dsfont,mathrsfs,bbm}
\usepackage{graphicx,float,psfrag,epsfig,amssymb}
\usepackage[usenames,dvipsnames,svgnames,table]{xcolor}
\definecolor{darkgreen}{rgb}{0.0,0,0.9}
\definecolor{darkblue}{rgb}{0.0, 0.0, 0.55}
\definecolor{darkpowderblue}{rgb}{0.0, 0.2, 0.6}
\usepackage[pagebackref,letterpaper=true,colorlinks=true,pdfpagemode=none,citecolor=darkblue,linkcolor=BrickRed,urlcolor=BrickRed,pdfstartview=FitH]{hyperref}

\usepackage{url}   
\usepackage{wrapfig}
\usepackage{relsize}
\usepackage{color}
\usepackage{pict2e}
\usepackage{caption}
\usepackage{nameref}
\usepackage{makecell}
\usepackage[font={small}]{caption} 
\usepackage{subcaption}
\usepackage{mathtools}
\usepackage{natbib}
\let\chapter\section
\usepackage[ruled,vlined]{algorithm2e}
\usepackage[noend]{algorithmic}

\DeclareMathAlphabet{\mathpzc}{OT1}{pzc}{m}{it}
\newtheorem{propo}{Proposition}[section]
\newtheorem{lemma}[propo]{Lemma}

\newtheorem{coro}[propo]{Corollary}
\newtheorem{thm}[propo]{Theorem}
\newtheorem{remark}[propo]{Remark}

\usepackage{hyperref}       
\usepackage{url}   

\addtocontents{toc}{\protect\setcounter{tocdepth}{2}}

 \input{notation}

\def\AR{\mathsf{AR}}
\def\SR{\mathsf{SR}}

\def\SR{\mathsf{SR}}
\def\AR{\mathsf{AR}}
\def\BR{\mathsf{BR}}
\def\tmu{\widetilde{\mu}}
\def\vph{\varphi}
\def\tW{\widetilde{W}}

\newcommand{\indep}{\perp \!\!\! \perp}

\title{Adversarial robustness for latent models: Revisiting the robust-standard accuracies tradeoff}

\author{ 
Adel Javanmard\thanks{Data Sciences and Operations Department, University
of Southern California}  \and 
Mohammad Mehrabi\footnotemark[1] \thanks{The names of the authors are in alphabetical order. }
}

\begin{document}

\maketitle

\begin{abstract}
Over the past few years, several adversarial training methods have been proposed to improve the robustness of machine learning models against adversarial perturbations in the input. Despite remarkable progress in this regard, adversarial training is often observed to drop the standard test accuracy. This phenomenon has intrigued the research community to investigate the potential tradeoff between standard accuracy (a.k.a generalization) and robust accuracy (a.k.a robust generalization) as two performance measures. In this paper, we revisit this tradeoff for latent models and argue that this tradeoff is mitigated when the data enjoys a low-dimensional structure. In particular, we consider binary classification under two data generative models, namely Gaussian mixture model and generalized linear model, where the features data lie on a low-dimensional manifold. We develop a theory to show that the low-dimensional manifold structure allows one to obtain models that are nearly optimal with respect to both, the standard accuracy and the robust accuracy measures. We further corroborate our theory with several numerical experiments, including Mixture of Factor Analyzers (MFA) model trained on the MNIST dataset.
\end{abstract}

\section{Introduction}
We are witnessing an unparalleled growth of machine learning tools in various applications domain, where these tools are deployed to inform decisions that directly impact human’s lives, from health interventions to credit decisions, sentencing and autonomous driving. Given the safety-critical nature of these applications, reliability and robustness of machine learning systems have become one of the paramount goals of today's AI.

Robust estimation has been one of the central topics in statistics, notably by the seminal work of Tukey \citep{tukey1960survey}, Huber \citep{huber1992robust}, and Hampel \citep{hampel1968contributions}, among others. The majority of work in this area has focused on robustness with respect to outliers (a small fraction of predictors/and or response variables which are contaminated by gross errors.) 

Another relevant notion that has spurred a surge of interest in recent years  is that of \emph{adversarial robustness}. While machine learning models, and deep learning in particular, have shown remarkable empirical performance, many of these models are known to be highly vulnerable to adversarially chosen perturbations to the input data at test time, known as \emph{adversarial attacks}. Even more surprisingly, many of such adversarial attacks can be designed to be slight modifications of the input which are seemingly innocuous and imperceptible. For example, in image processing and video analysis there are several examples of adversarial attacks in form of indiscernible  pixel-wise perturbations which can significantly degrade the performance of the state-of-the-art classifiers~\citep{szegedy2013intriguing,biggio2013evasion}. Other examples include well-designed malicious contents like malware which can pass the scanning classifiers and yet harm the system, or adversarial attacks on speech recognition systems, such as GoogleNow or Siri, which are incomprehensible or even completely inaudible to human and can still control the virtual assistant software \citep{carlini2016hidden,vaidya2015cocaine,zhang2017dolphinattack}.

In response to this fragility, a growing body of work in the past few years has sought to improve the robustness of machine learning systems against adversarial attacks.  Despite remarkable progress in designing robust training algorithms and certifiable defenses, it is often observed that these methods compromise the statistical accuracy on unperturbed test data (i.e., test data drawn form the same distribution as training data). Such observation had led prior work to speculate a tradeoff between the two fundamental notions of \emph{robustness} and \emph{generalization} (for a non-exhaustive list see e.g, \citep{DBLP:conf/iclr/MadryMSTV18,raghunathan2019adversarial,min2020curious,mehrabi2021fundamental}). For example, the highest obtained 
$\ell_\infty$-robust accuracy on CIFAR10 (without using additional data) with $\eps_\infty = 8/255$ is 60\%, with standard accuracy of 85\% (which is 10\% less than state-of-the-art standard accuracy for $\eps_\infty = 0$).

Some of the promising adversarial training methods, such as TRADES~\citep{DBLP:conf/icml/ZhangYJXGJ19} acknowledge such tradeoff by including a regularization parameter which allows to tune between these two measures of performance. There has been also recent line of work \citep{javanmard2020precise,pmlr-v125-javanmard20a} which provides precise asymptotic theory for this tradeoff and how it is quantitively shaped by different components of the learning problem (e.g, adversary's power, geometry of perturbations set, overparamterization, noise level in training data, etc.) For the setting of linear regression and binary classification it is proved that there is an inherent tradeoff between robustness and standard accuracy (generalization) which holds at population level and for any (potentially computationally intensive) training algorithms~\citep{pmlr-v125-javanmard20a,dobriban2020provable,mehrabi2021fundamental}. Nonetheless, these work make strong assumptions on the distribution of data (e.g, Gaussian or Gaussian mixture models), which fail to capture various natural structures in data.  This stimulates the following tantalizing question:  
\begin{quote}
\emph{(*) Are there natural data generative models under which the tradeoff between robustness and the standard accuracy (generalization) vanishes, in the sense that one can find models which are performing well (or even optimal) with respect to both measures?}
\end{quote}
As a step toward answering this question, \cite{NEURIPS2020_61d77652} show that when data is well separated, there is no inherent conflict between standard accuracy and robustness. It also provides numerical experiments on a few image datasets to argue that these data are indeed $r$-well separated for some value $r$ larger than the perturbation radii used in adversarial attacks (i.e., data from different classes are at least $r$ distance  apart in the pixel domain.) In~\citep{xing2021adversarially}, adversarially robust estimators are studied for the setup of linear regression and a lower bound on their statistical minimax rate is derived. The minimax rate  lower bound for sparse model is much smaller than the one for
dense model, whereby~\cite{xing2021adversarially} argues the importance of incorporating sparsity structure in improving robustness. 

The current work takes another perspective towards question (*) by considering the low-dimensional manifold structures in data. Many high-dimensional real-world data sets enjoy low-dimensional structures, and learning low-dimensional representations of raw data is a common task in information processing. In fact, the entire field of dimensionality reduction and manifold learning has been developed around this task. To give concrete examples, the MNIST database of handwritten digits consists of images of size $28\times 28$ (i.e., ambient dimension of 784), while its intrinsic (manifold)  dimension is estimated to be $\approx 14$, based on local neighborhoods of data. Likewise, the  CIFAR10 database consists of color images of size $32 \times 32$ (i.e. ambient dimension of $3,072$), but its intrinsic dimension is estimated to be $\approx 35$~\citep{costa2004learning,rozza2012novel,spigler2020asymptotic}.
The high-level message of the current work is that the low-dimensional structures in data can mitigate the tradeoff between standard accuracy and robustness, and potentially enable training models that perform gracefully (or even optimal) with respect to both measures.

\subsection{Summary of contributions} 
In this work we focus on two widely used models for binary classification, namely Gaussian-mixture model and the generalized linear model, where we also assume that the feature vectors lie on a $k$-dimensional manifold in a $d$-dimensional space ($k<d$). We consider adversarial setting with norm-bounded perturbation (in $\ell_p$ norm), for general $p\ge 2$.

We use the minimum nonzero singular value of the `lifting matrix' $W\in\reals^{d\times k}$ (between the manifold and the ambient space) as a measure of low-dimensional structure of data; cf.~\eqref{eq: gaussian-mix} and \eqref{eq: generative}. We assess the generalization property of a model through the notion of \emph{standard risk}, and its robustness against adversarial perturbation through the notion of \emph{adversarial risk} (See Section~\ref{sec:formulation} for formal definition.) Our main contributions are summarized as follows:
\begin{itemize}
\item Under both data generative models, we derive the Bayes-optimal estimators, which provably attain the minimum standard risk. We prove that as long as $\sigma_{\min}(W)$ diverges as $d\to \infty$, with a growth rate that depends on the adversary's power $\eps_p$ and the perturbation norm $\ell_p$, then the Bayes-optimal estimator asymptotically achieves the minimum adversarial risk as well. This implies that the tradeoff between robustness and generalization asymptotically disappears as data becomes more structured.
\item While the gap between the optimal standard risk and optimal adversarial risk shrinks for data with low-dimensional structure, we show that these two risk measures (as functions of estimators) stay away from each other. Specifically, we come up with an estimator for which the standard and adversarial risks remain away from each other by a constant $c>0$ independent of $k, d$.

\item In Section~\ref{sec:agnostic}, we consider an adversarial training method based on robust empirical loss minimization. While this algorithm is structure agnostic we provably show that it results in models that are robust and also generalize well. Note that the data structures (distribution), even if not deployed by the training procedure, still comes into picture as the adversarial risk and standard risk are defined with respect to this data distribution.

\item We corroborate our theoretical findings with several synthetic simulations. We also train Mixture of Factor Analyzers (MFA) models on the MNIST image dataset. This results in low-rank models from which we can generate new images. Furthermore, the Bayes-optimal classifier can be precisely computed for the MFA model. We show empirically that as the ratio of ambient dimension to the rank diverges (data becomes more structured) the gap between standard risk and adversarial risk vanishes for the Bayes-optimal classifier. In other words, Bayes-optimal estimator becomes optimal with respect to both risks.     

\end{itemize}

\subsection{Related work}    
There is a growing body of work on the tradeoff between robustness and generalization (see e.g.,~\citep{tsipras2018robustness,madry2017towards,DBLP:conf/icml/ZhangYJXGJ19,raghunathan2019adversarial,NEURIPS2020_61d77652,min2020curious,mehrabi2021fundamental}). In particular,~\cite{dobriban2020provable} consider the isotropic Gaussian-mixture model with two and three classes, and derive Bayes-optimal robust classifiers for $\ell_2$ and $\ell_\infty$ adversaries. This work proves a tradeoff between standard and robust risks which becomes bolder when the classes are imbalanced. 

The prior work~\citep{jalal2017robust,NEURIPS2018_8cea559c,stutz2019disentangling} proposed the concept of on-manifold attack, where the adversarial perturbations are done in the latent low-dimensional space. In~\citep{stutz2019disentangling}, it is argued that on-manifold adversarial examples are acting as generalization error and adversarial training against such attacks improve the generalization of the model as well. In addition, a so-called on-manifold adversarial training (based on minimax formulation) has been proposed which is similar to the adversarial training method of \cite{madry2017towards} but tailored to perturbations in the manifold space. The subsequent work~\citep{NEURIPS2020_23937b42} proposes dual manifold adversarial training (DMAT) method which considers adversarial perturbations in both the manifold and the image space to robustify models against a broader class of adversarial attacks. In this terminology, in our current work we consider out-of-manifold perturbations (in the ambient space). Also let us emphasize that~\citep{stutz2019disentangling,NEURIPS2020_23937b42} are based on empirical studies on image databases and more on an algorithmic front. The current work contributes to this literature by developing a theory for the role of manifold structure of data in the interplay between robustness and generalization, under specific binary classification setups (viz. Gaussian-mixture model and generalized linear model)

\section{Problem Formulation}\label{sec:formulation}
In this section we discuss the problem setting and formulation of this paper in greater detail. After adopting some notations, we give a brief overview of adversarial setting and describe two data generative models, namely the Gaussian mixture models (GMMs) and generalized linear models (GLMs), which incorporate latent low-dimensional manifold structure. We then conclude this section by a short background on the Bayes-optimal binary classifiers.
\smallskip

\noindent {\bf Notations.} For a matrix $W\in \reals^{d\times k} $, let $||W||$ denote its operator norm, $W^{\dagger}$ stand for the Moore–Penrose inverse, and $\sigma_{\min}(W)$ denote its smallest ``nonzero'' singular value. For a vector $x\in \reals^d$ and $p\geq1$, we define the $\ell_p$ norm $\pnorm{x}=\left(\sum_{i=1}^d x_i^p\right)^{1/p}$. 
In addition, let $B_\eps(x)$ denote the $\ell_2$-ball centered at $x$ with radius $\eps$. Throughout the paper, for two functions $f,g$ from integers to positive real numbers, we say $f(d)=o_d(g(d))$, as $d$ grows to infinity, if for every $\delta>0$, we can find a positive integer $\ell$ such that for $d\geq \ell$, we have $f(d)/g(d)\leq \delta$. In addition, let $\normal(\mu,\Sigma)$ denote the probability density of a multivariate normal distribution with mean $\mu$ and covariance $\Sigma$.   

\subsection{Adversarial setting}
In the binary classification problem, we are given a set of labeled data points $\{(x_i,y_i)\}_{i=1:n}$ which are drawn i.i.d. from a common law $\mathcal{P}$, where $x \in \cX\subset\reals^d$ is the feature vector and $y\in \{+1,-1\}$ is the label associated to the feature $x$. The goal is to predict the label of  a new test data point with a feature vector drawn from the similar population. To this end,  the learner tries to fit a binary classification model to the training set, which results in an estimated model $\widehat{h}:\cX\rightarrow \{-1,+1\} $. The conventional metric to measure the accuracy of a classifier $h$ is its average error probability on an unseen data point $(x,y)\sim \mathcal{P}$.  This is often referred to as the \textit{standard risk} of the classifier, a.k.a. generalization error. Concretely, standard risk of a classifier $h$ is defined as the following:
\begin{equation} \label{eq: SR}
\SR(h):= \prob\left(h(x)y\leq 0\right)\,.
\end{equation}
 Despite the remarkable success in deriving classifiers with high accuracy (low standard risk) during the past decades, it has been observed that even the state-of-the-art classifiers are vulnerable to minute but adversarially chosen perturbations on test data points. 
 
 The adversarial setting is often formulated as a game between the learner and the adversary. Given access to unperturbed training data, the learner fits a model $h: \cX \to \{-1,+1\}$. After observing the model $h$ and each  test data point $(x,y)$ generated from the distribution $\prob$, the adversary perturbs the data point arbitrarily as far as its within its budget.  
A common and widely-used adversarial model is that of norm-bounded perturbations, where for each test data point $(x,y)$ the adversary chooses an arbitrary perturbation $\delta$ from the $\ell_p$ ball of radius $\eps_p$ and replaces $x$ by $x+\delta$. Here, $\eps_p$ is a parameter of the setting which quantifies adversary's power.\footnote{We will drop the index $p$ and write $\eps$ for the adversary's power when it is clear from the context.}  
%
%
%
The \textit{adversarial risk} of the classifier $h$ is defined as the following:
\begin{equation}\label{eq: AR-0}
\AR(h)= \E_{(x,y)\sim\mathcal{P}}\left( \sup_{\pnorm{\delta}\le\eps_p}\ell(h(x+\delta),y) \right)\,,
\end{equation}
for some loss function $\ell$. For the 0-1 loss $\ell(s,t) = \ind(st\le 0)$, this measure amounts to

\begin{equation}\label{eq: AR}
\AR(h)= \prob\left( \inf\limits_{\pnorm{x'-x} \leq \eps_p}^{} h(x')y \leq 0  \right)\,.
\end{equation}
\begin{remark}
A couple points are worth noting regarding the adversarial setting:
\begin{itemize}
\item The adversary chooses perturbation ``after'' observing the test data point. The perturbation $\delta$ can in general depend on $x$, i.e. different data points can be perturbed differently. Therefore, in the definition~\eqref{eq: AR-0}, the supremum is taken inside the expectation.
\item In the above setting, the perturbations are added in the test time, while the learner is given access to unperturbed training data. Other adversarial setups are also studied in the literature; see e.g.~\cite{}, where an attacker can observe and modify all training data samples adversarially so as to maximize the estimation error caused by his attack.     
\item Another popular adversarial model is the so-called distribution shift. In this model, in contrast to norm bounded perturbations as discussed above, the adversary can shift the test data distribution. The adversary's power is measured  in terms of the Wasserstein distance between the test and the training distributions; see ~\cite{staib2017distributionally,dong2020adversarial,pydi2020adversarial,mehrabi2021fundamental} for a non-exhaustive list of references. That said, our focus in this paper is on the norm-bounded perturbations.
\end{itemize} 
\end{remark}

From the definition of standard risk and adversarial risk given by~\eqref{eq: SR} and \eqref{eq: AR}, it can be seen that the adversarial risk is always at least as large as the standard risk. We refer to the non-negative difference of adversarial risk and standard risk as the \textit{boundary risk}, formulated by
\begin{align}
\BR(h)&:=\AR(h)-\SR(h)\nonumber\\
&=\prob \left( h(x)y\geq 0,  \inf\limits_{\pnorm{x'-x}\leq \eps_p}^{} h(x)h(x') \leq 0 \right)\label{eq: BR}\,.
\end{align}
The boundary risk can be considered as the average vulnerability of the classifier with respect to small perturbations on successfully labeled data points. In other words, it measures the likelihood that the classifier correctly determines the label of a data point, but fails to label another test input very close to the primary data point. In the main result section, we study the boundary risk of optimal classifiers (having the lowest standard risk) in scenarios that features vectors lie on a low-dimensional manifold. We next discuss the data generative models. 

\subsection{Data generative model}\label{sec:data-model}

\noindent{\bf Latent low-dimensional manifold models.} We focus on the binary classification problem with high-dimensional features generated from a low-dimensional latent manifold. Specifically, we assume that for the features vector $x\in \reals ^d$, and the binary label $y\in \{+1,-1\}$, there exists an inherent low-dimensional link $z\in \reals^k$ such that $x\indep y|z$. This structure can be perceived as a transformed binary classification model, where low-dimensional features $z\in \reals ^k$ of a hidden classification problem with labels $y\in \{+1,-1\}$, are embedded in a high-dimensional space by a mapping $G: \reals^k\rightarrow \reals^d$. The learner observes the embedded high-dimensional features $x_i=G(z_i)$ and the primary binary labels $y_i$, while being oblivious to the low-dimensional latent vector $z_i$.



 Throughout the paper, we consider a special case of this model, where $G(z)=\vph(Wz)$ with $W\in\reals^{d\times k}$ a tall full-rank weight matrix, and $\vph$ acting entry-wise on vector inputs with a derivative $d\vph/dt\geq c$, for some positive constant $c>0$.

\bigskip

\noindent{\bf Classification settings.}
The focus of this paper is on two widely used binary classification settings: (i) Gaussian mixture models (ii) generalized linear models which we briefly explain below. 

\begin{itemize}
\item {\bf Gaussian mixture models.}
 In the Gaussian mixture model, the binary response value $y$ accepts the positive label with probability $\pi$, and the negative label with probability $1-\pi$. In this setting, labels are assigned independently from the feature vector $z$, while feature vectors are generated from a multivariate Gaussian distribution with the mean vector $y\mu$, and a certain covariance matrix.
Concretely, the data generating law for the Gaussian mixture problem with features coming from a low-dimensional manifold can be written as the following: 
\begin{equation}\label{eq: gaussian-mix}
 y\sim{\rm Bern}(\pi, \{+1,-1\}),~~  x=\vph(Wz), ~~ z\sim\normal(y\mu,I_k)\,.
\end{equation}
 In this model, we consider low-dimensional isotropic Gaussian features. In other words, manifold features $z$ are drawn from a Gaussian distribution with identity covariance matrix.   

\item {\bf Generalized linear models.}
In binary classification under a generalized linear model, there is an increasing function $f:\reals: \rightarrow [0,1]$, a.k.a. link function,  along with a linear predictor $\beta\in \reals ^k$, where the score function $f(z^\sT\beta)$ denotes the likelihood of feature vector $z$ accepting the positive label. Formally, the data generating law for this classification problem under the low-dimensional manifold model can be formulated as the following: 
\begin{equation}\label{eq: generative}
y=\begin{cases}
+1& {\rm w.p. } ~f(z^\sT\beta)\,,\\
-1& {\rm w.p.  } ~ 1-f(z^\sT\beta)\,. \\ 
\end{cases},~ x=\vph(Wz),~z\sim\normal(0,I_k)\,.
\end{equation}  
Popular choices of the link function $f$ are the logistic model $f(t)=1/(1+\exp(-t))$, and the probit model $f(t)=\Phi(t)$ with $\Phi(t)$ being the standard normal cumulative distribution function.
\end{itemize}
\subsection{Background on optimal classifiers}
For each classification setup described in the previous section, we want to identify the classifiers that are optimal with respect to the standard risk. To this end, we provide a summary of the Bayes-optimal classifiers. For a data point $(x,y)\sim \mathcal{P}$, consider the conditional distribution function $\eta(x):=\prob( y=+1|X=x)$. This distribution function can be perceived as the likelihood of assigning the positive label to a data point with feature vector $x$. The Bayes-optimal classifier simply assigns label $y=+1$ to the feature vector $x$, if for this feature there is a higher likelihood to accept the label $+1$ than $-1$. In other words, $h_{\rm Bayes}(x)=\sign(\eta(x)-1/2)$. The next proposition states the optimality of the Bayes-optimal classifier. 

\begin{propo}\label{propo: bayes}
Among all the classifiers $h: \reals^d \rightarrow \{+1,-1\}$, such that $h$ is a Borel function, the Bayes-optimal classifier $h_{\rm Bayes}(x)=\sign(\eta(x)-1/2)$ has the lowest standard risk.
 \end{propo}

Proof of Proposition \ref{propo: bayes} is provided in Section~\ref{sec:proofs}. The next corollary uses Proposition \ref{propo: bayes} to characterize the Bayes-optimal classifier under each of the binary classification settings described earlier in Section \ref{sec:data-model}.
\begin{coro}\label{coro: optimals}
 Under the Gaussian mixture model \eqref{eq: gaussian-mix}, the Bayes-optimal classifier can be formulated by
$$h^*(x)=\sign\left(\vph^{-1}(x)^\sT\left(WW^\sT\right)^{\dagger}W\mu-q/2\right)\,,$$
with $q=\log(\frac{1-\pi}{\pi})$. Moreover, under the generalized linear model \eqref{eq: generative}, the Bayes-optimal classifier is given by 
\begin{equation*}
h^*(x)=\sign\left(f\left(\beta^{\sT}(W^\sT W)^{-1}W^{\sT}\vph^{-1}(x)\right)-1/2\right)\,.
\end{equation*}
\end{coro}
It is worth noting that in the described manifold latent model of Section \ref{sec:data-model}, the weight matrix $W$ is tall and full-rank, and $\vph$ is strictly increasing hence both $W^\sT W$  and $\vph$ are invertible. 
\section{Main results}

We will focus on the described binary classification settings of Section \ref{sec:data-model}. In each setting, we characterize the asymptotic behavior of the associated boundary risk of Bayes-optimal classifiers, when the ambient dimension $d$ grows to infinity. We aim at studying the role of low-dimensional latent structure of data in obtaining a vanishing boundary risk for the Bayes-optimal classifiers. In this case, the Bayes-optimal classifiers are optimal with respect to both measures of standard accuracy and the robust accuracy.
 
\subsection{Gaussian mixture model}
Consider the Gaussian mixture model with features lying on a low-dimensional manifold, cf. \eqref{eq: gaussian-mix}. Recall that the learner only observes the ambient $d-$dimensional features $x$, and is oblivious to the original $k$-dimensional manifold features $z$. The next result states that under this setup, the boundary risk of the Bayes-optimal classifier will converge to zero, when the  minimum nonzero singular value of the weight matrix $W$ grows at a sufficient rate, which depends on adversary's power $\eps_p$ and the choice of perturbations norm $\ell_p$. 

\begin{thm}\label{thm: Gaussian-mix} Consider the binary classification problem under the Gaussian mixture model \eqref{eq: gaussian-mix} in the presence of an adversary with $\ell_p$-norm bounded adversary power $\eps_p$, for $p\geq 2$. By letting the ambient dimension $d$ grow to infinity, under the condition that the weight matrix $W$ satisfies
\begin{align}\label{eq:W-cond}
\frac{\eps_pd^{\frac{1}{2}-\frac{1}{p}}}{\sigma_{\min}(W)}=o_d(1)\,,
\end{align}
  the boundary risk of the Bayes-optimal classifier converges to zero.
\end{thm}
The proof of Theorem \ref{thm: Gaussian-mix} is given in Section~\ref{proof: thm: Gaussian-mix}.

We proceed by discussing condition~\eqref{eq:W-cond}. As $\eps_p$ gets larger, the condition becomes more strict which is expected; larger value of $\eps_p$ indicates a stronger adversary which makes the boundary risk larger. In addition, $\sigma_{\min}(W)$ somewhat measures the extent of low-dimensional structure in data; small $\sigma_{\min}(W)$ indicates that there are directions in the low-dimensional space along which the energy of the signal is not scaled sufficiently large when transformed into the ambient space. Therefore, the adversary can perturb feature $x$ along those dimensions as the existent signal is weak. Finally, since $\pnorm{\delta}\le \twonorm{\delta}$ for $p\ge 2$, an adversary with power $\eps$ in $\ell_p$ norm is stronger than an adversary with power $\eps$ in $\ell_2$ norm. This is consistent with the fact that $d^{\frac{1}{2}-\frac{1}{p}}$ is increasing in $p$ and so the condition becomes stronger for larger $p$.

\bigskip

\noindent{\bf Example.}
Consider the case of $\varphi(\cdot)$ being the identity function and $p=2$. We observe that for feature $x$ with label $y$, $x_i \sim \normal(yw_i^\sT\mu, \twonorm{w_i}^2)$. To be definite, we fix $\twonorm{w_i}=1$, which implies in particular $\fronorm{W}^2 = d$. To simplify further, we assume that all the non-zero
singular values of $W$ to be equal, whence  $W^\sT W = (d/k) I$. In this case, condition~\eqref{eq:W-cond} reduces to $\eps_2 = o(\sqrt{d/k})$. In particular, if $\eps_2 = O(1)$ and the dimension ratio $d/k \to \infty$ the boundary risk converges to zero.



Figure \ref{fig:gaussian}  validates the result of Theorem \ref{thm: Gaussian-mix} under the Gaussian mixture model \eqref{eq: gaussian-mix} with $\pi=1/2$, $\mu=\normal(0,I_k/k)$, in the presence of an adversary with $\ell_2$ bounded adversarial attacks of power $\eps_2$. In this example, we fix the high-dimensional feature dimension $d=300$, and vary the dimensions ratio $d/k$ from $1$ to $300$. Further, we consider the identical function $\vph(t)=t$, and let the feature matrix $W$ have independent Gaussian entries $\normal(0,1/k)$.  Figure \ref{fig:gaussian-1} shows the effect of dimensions ratio $d/k$ on the standard risk, adversarial risk, and the boundary risk of the Bayes-optimal classifier. For each fixed values $(k,d)$, we generate $M=100$ independent realizations and compute the risks. The shaded area around each curve denotes one standard deviation (computed over $M$ realizations) above and below the average curve. As it can be seen, the boundary risk will eventually converge to zero. Finally, in Figure \ref{fig:gaussian-2}, we consider several values for adversary's power, where it can be observed that for all adversary's power $\eps_2$, the boundary risk decays to zero as the feature dimensions $d/k$ grows.

\begin{figure*}
	\centering
	\begin{subfigure}[t]{0.47\textwidth}
		\centering
		\includegraphics[scale=0.4]{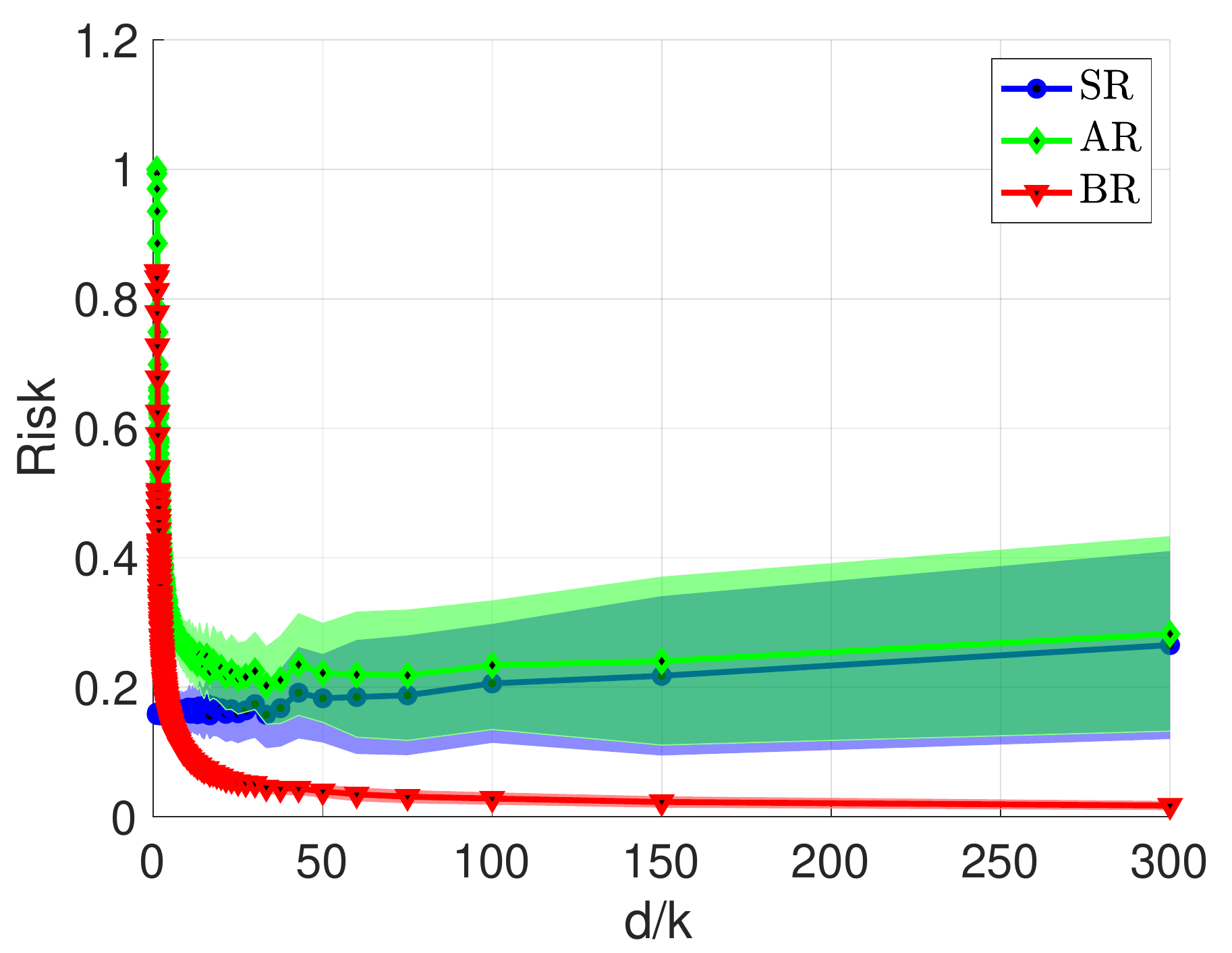}
		\caption{Behavior of the standard and  adversarial risks of the Bayes-optimal classifier for the adversary's power $\eps_2=1$.}
		\label{fig:gaussian-1}
	\end{subfigure}%
	\hfill
	\begin{subfigure}[t]{0.47\textwidth}
		\centering
		\includegraphics[scale=0.4]{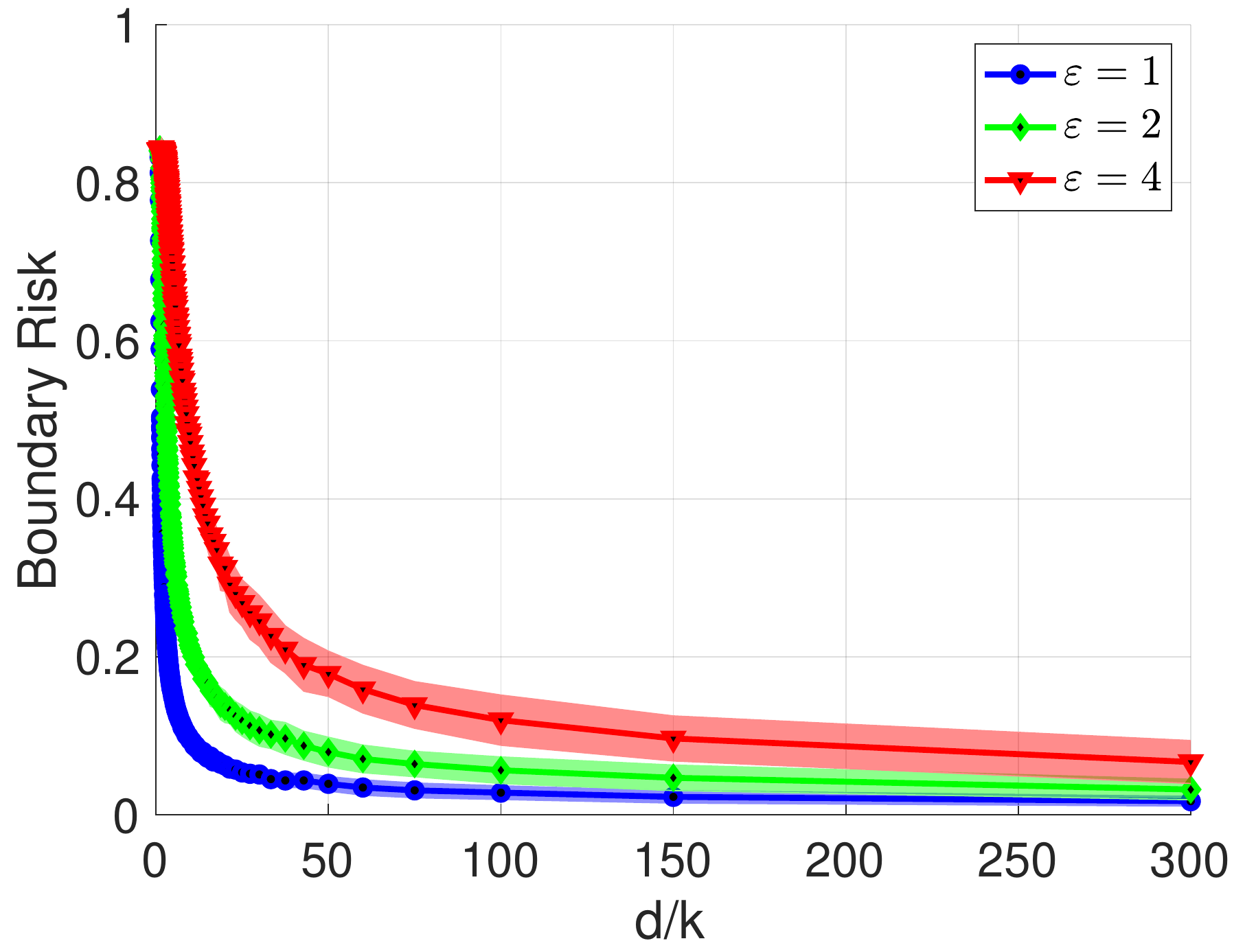}
		\caption{Behavior of the boundary risk of the Bayes-optimal classifier for the several values of adversary's power $\eps_2$.}
		\label{fig:gaussian-2}
	\end{subfigure}
	\caption{Effect of the dimensions ratio $d/k$ on the standard, adversarial, and the boundary risk of the Bayes-optimal classifier with $\ell_2$ perturbations under the Gaussian mixture model \eqref{eq: gaussian-mix}, where features lie on a low-dimensional manifold.
Solid curves represent the average values, and the shaded area around each curve represents one standard deviation above and below the computed average curve over the $M=100$ realizations.}
	\label{fig:gaussian}
\end{figure*}

In Theorem \eqref{thm: Gaussian-mix}, we showed that when features lie on a low-dimensional manifold, the Bayes-optimal classifier is also optimal with respect to the adversarial risk. In other words, the adversarial risk is always at least as large as the standard risk, for any classifier, and the gap between the ``minimum'' of these two risks converges to zero.  This result raises the blow natural question:

\emph{``Does the boundary risk of \emph{any} classifier vanish under the low-dimensional latent structure?''}

In the next proposition, we provide a simple example to show that such behavior (vanishing boundary risk) does not necessarily happen for all classifiers.

\begin{propo}\label{propo: example}
Consider the Gaussian mixture model \eqref{eq: gaussian-mix} with $\mu$ having i.i.d. $\normal(0,1/k)$ entries with class probability $\pi=1/2$ in the presence of an adversary with bounded $\ell_p$ perturbations of size $\eps_p$.  In addition, suppose that the rows of the feature matrix $W$ are sampled from the $k$-dimensional unit sphere and consider $\vph$ being the identity function.  Then, the boundary risk of the classifier $h(x)=\sign(e_1^\sT x)$ with $e_1=(1,0,0,...,0)$  is lower bounded by some constant $c_{\eps_p}$, where $c_{\eps_p}$ depends only on $\eps_p$ (independent of dimensions $k,d$), and is strictly positive for positive values of $\eps_p$.
\end{propo}
We refer to Section~\ref{sec:proofs} for proof of this proposition.
  


\subsection{ Binary classification under generalized linear models}
Consider a binary classification problem under a generalized linear model with features enjoying a low-dimensional latent structure, cf. \eqref{eq: generative}. 
The next result states that under certain conditions on the weight matrix, the boundary risk of the Bayes-optimal classifier will converge to zero, as the ambient dimension grows to infinity.

\begin{thm}\label{thm: generative_new}

Consider the binary classification problem under the generalized linear model \eqref{eq: generative} in the presence of an  adversary with $\ell_p$-bounded perturbations of power $\eps_p$ for some $p\geq 2$.  
Assume that as the ambient dimension $d$ grows to infinity, the weight matrix $W$ satisfies the following condition:
\[
\frac{\eps_p d^{\frac{1}{2}-\frac{1}{p}}}{\sigma_{\min}(W)}=o_d(1)\,.
\]
 Then the boundary risk of the Bayes-optimal classifier converges to zero. 
\end{thm}
The proof of Theorem \ref{thm: generative_new} is given in Section~\ref{sec:proofs}.

\begin{figure*}[t]
	\centering
	\begin{subfigure}[t]{0.47\textwidth}
		\centering
		\includegraphics[scale=0.4]{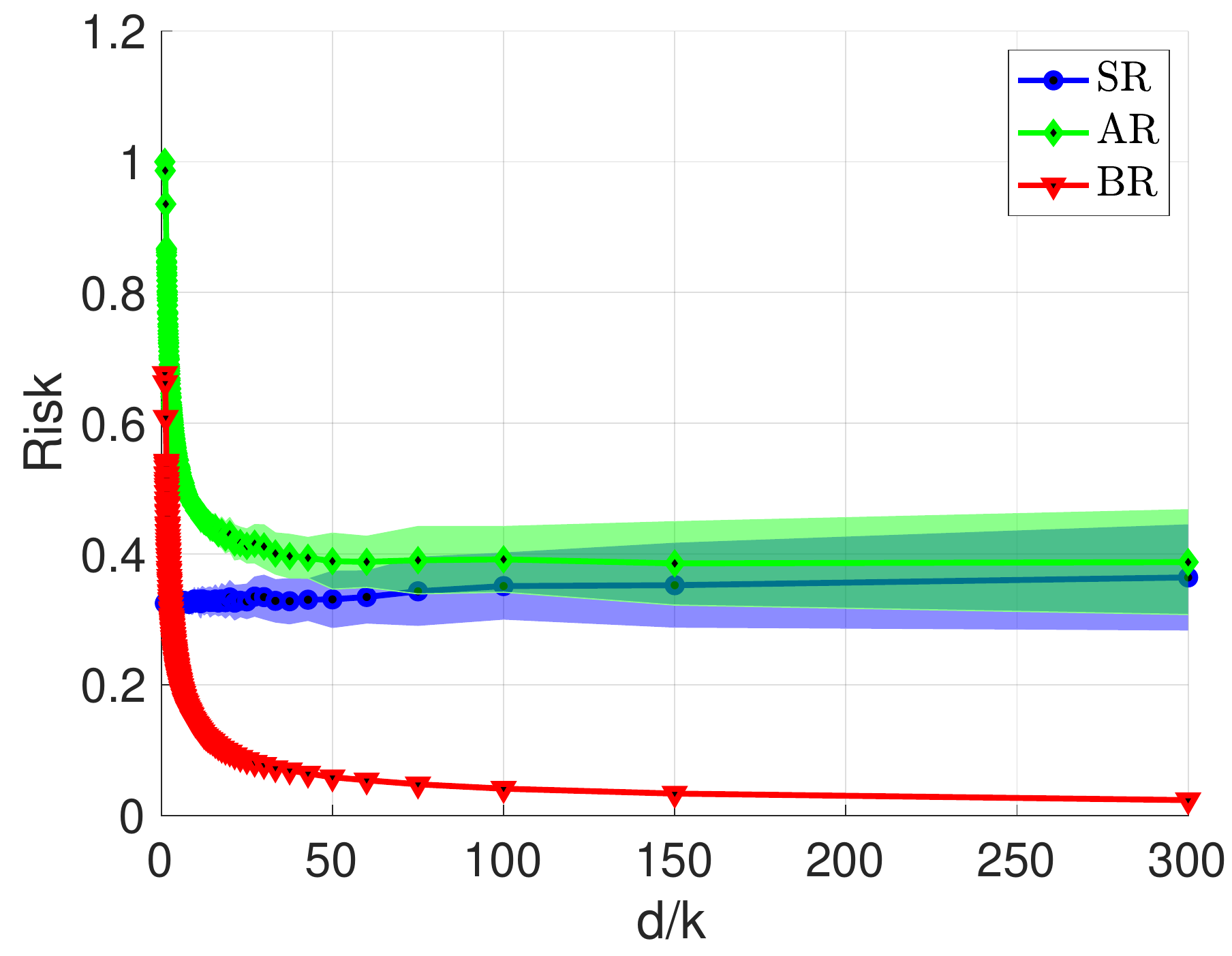}
		\caption{Behavior of the standard and  adversarial risks of the Bayes-optimal classifier for the adversary's power $\eps_2=1$. }
		\label{fig:logistic-1}
	\end{subfigure}%
	\hfill
	\begin{subfigure}[t]{0.47\textwidth}
		\centering
		\includegraphics[scale=0.4]{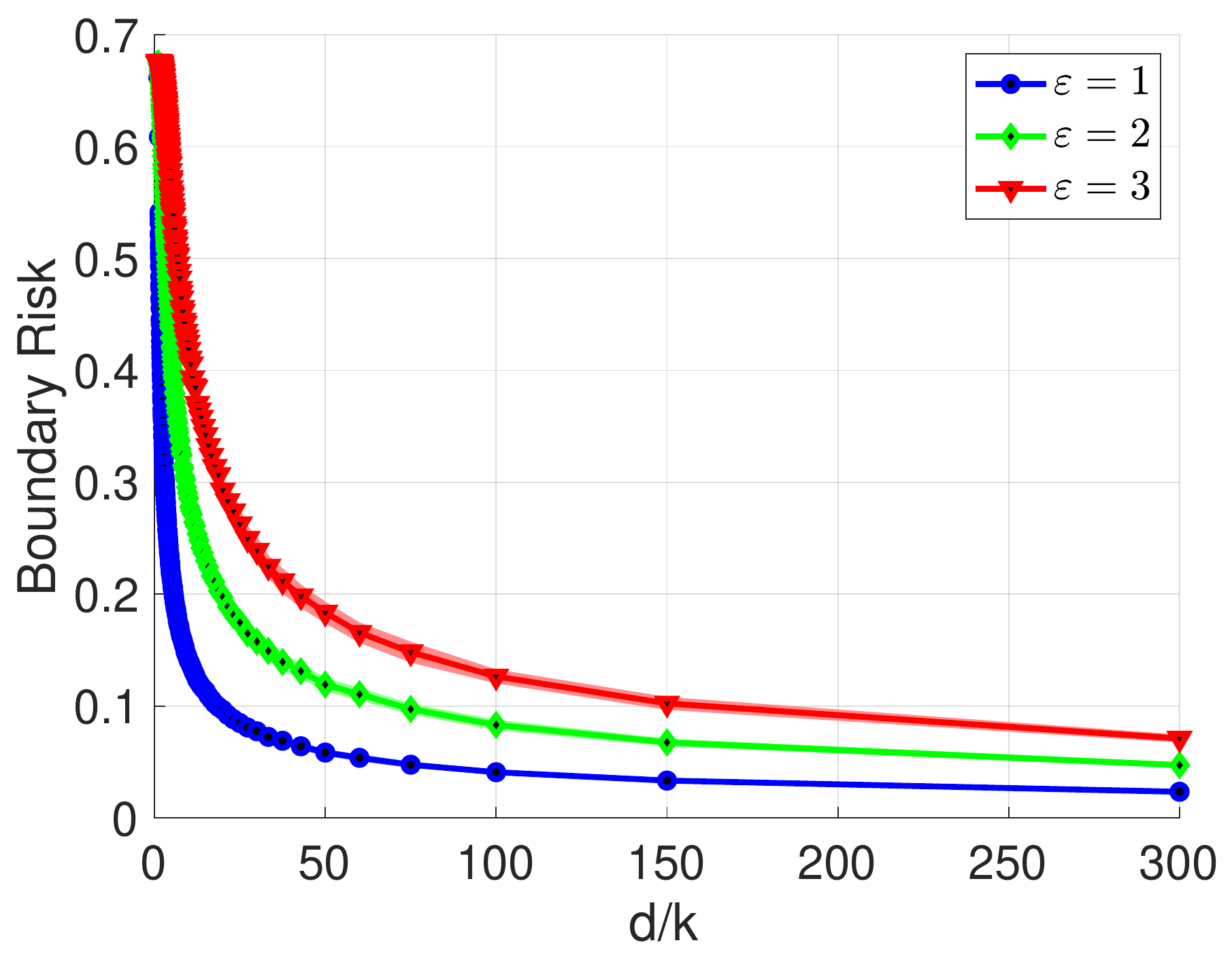}
		\caption{Behavior of the boundary risk of the Bayes-optimal classifier for the several values of adversary's power $\eps_2$.}
		\label{fig:logistic-2}
	\end{subfigure}
	\caption{Effect of the dimensions ratio $d/k$ on the standard, adversarial, and the boundary risk of the Bayes-optimal classifier of the generalized linear model \eqref{eq: generative} with $\ell_2$ perturbations, in which features are coming from a low-dimensional manifold. 
Solid curves represent the average values, and the shaded areas represent one standard deviation above and below the corresponding curves over $M=100$ realizations.}
	\label{fig:logistic}
\end{figure*}

Figure \ref{fig:logistic} validates the result of Theorem \ref{thm: generative_new} for binary classification under the generalized linear model \eqref{eq: generative} with identity $\vph$ mapping and $\ell_2$ perturbations. In this example, the ambient dimension $d$ is fixed at $300$, and the manifold dimension $k$ varies from $1$ to $300$. In addition, the linear predictor $\beta$ and the weight matrix $W$ have i.i.d.  $\normal(0,1/k)$ entries. For each fixed values $(k,d)$, we generate $M=100$ independent realizations, and we compute the average and the standard deviation of total $M$ obtained values. In each figure, the shaded areas are obtained by moving the average values one standard deviation above and below. Figure \ref{fig:logistic-1} denotes the behavior of the standard risk, adversarial risk, and the boundary risk of the Bayes-optimal classifier, as the dimensions ratio $d/k$ grows. Further, Figure \ref{fig:logistic-2} exhibits a similar behavior for several values of adversary's power $\eps$, in which it can be observed that the boundary risk decays to zero.

\subsection{Is it necessary to learn the latent structure to obtain a vanishing boundary risk? A simple case}\label{sec:agnostic}

In the previous sections, for the two binary classification settings, we showed that when the features inherently have a low-dimensional structure, the boundary risk of the Bayes-optimal classifiers will converge to zero, as the ambient dimension grows to infinity. A closer look at the Bayes-optimal classifier of each setting (can be seen in Corollary \ref{coro: optimals}) reveals the fact that these classifiers directly use the knowledge of the nonlinear mapping from the low-dimensional manifold to the ambient space. In other words, the Bayes-optimal classifiers explicitly draw upon the generative components $\vph$ and $W$. In this section, we investigate the existence of classifiers that are agnostic to the mapping between the low-dimensional and the high-dimensional space, while they have asymptotically vanishing boundary risk. For this purpose, consider binary classification under the Gaussian mixture model \eqref{eq: gaussian-mix}. In addition, assume a training set $\{(x_i,y_i )\}_{i=1}^n$ sampled from \eqref{eq: gaussian-mix}.  We focus on the class of linear classifiers $h_{\th}(x)=\sign(x^\sT \th)$ with $\th\in \reals^d$ and $\ell_2$ perturbation ($p=2$). 

We consider the logistic loss $\ell(t)=\log(1+\exp(-t))$, and assume that the adversary's power is bounded by $\eps$. We consider the minimax approach of \cite{madry2017towards} to adversarially train a model $\theta$ by solving the following robust empirical risk minimization (ERM):  
 \[
 \hth^{\eps}=\arg\min\limits_{\th\in {\reals^d}}^{} \frac{1}{n}\sum\limits_{i=1}^{n}\max\limits_{u\in B_{\eps}(x_i)}^{} \ell(y_iu^\sT\th)\,.
 \]
This is a convex optimization problem, as it can be cast as a point-wise maximization of the convex functions $\ell(y_iu^\sT\th)$. Further, when perturbations are from the $\ell_2$ ball, the inner maximization problem can be solved explicitly (see e.g.~\cite{javanmard2020precise}), which leads to the following equivalent problem:
\begin{equation}\label{eq: tmp9}
 \hth^{\eps}=\arg\min\limits_{\th\in {\reals^d}}^{} \frac{1}{n}\sum\limits_{i=1}^{n} \ell(y_ix_i^\sT\th-\eps\twonorm{\th})\,.
\end{equation}
Figure \ref{fig:main} demonstrates the effect of the dimensions ratio $d/k$ on the standard, adversarial, and the boundary risk of the classifier $h_{\hth^{\eps}}$ for four different choices of the feature mapping $\vph$: $(i)$ $\vph_1(t)=t$, $(ii)$ $\vph_2(t)=t/4+\sign(t) 3t/4$, $(iii)$ $\vph_3(t)=t+\sign(t)t^2$, and $(iv)$ $\vph_4(t)=\tanh(t)$.  In this example, we consider the ambient dimension $d=100$, and number of samples $n=300$. In addition, $k$ varies from $1$ to $100$, and $\mu$, $W$ have i.i.d. entries $\normal(0,1/k)$. Further, we consider balanced classes (each label $\pm1$ occurs with probability $\pi=1/2$). The plots in Figure \ref{fig:main} exhibit the behavior of the standard, adversarial, and the boundary risks of the classifier $h_{\hth^{\eps}}$, for each of these mappings and for the adversary's power $\eps = 1$. For each fixed value of $k,d$, we consider $M=20$ trials of the setup. The solid curve denote the average values over these $M$ trials. The shaded areas are obtained by plotting one standard deviation above and below the main curves. The plots in Figure~\ref{fig:app-2} showcase the boundary risk for different choices of $\eps$.  As we observe, the boundary risk decreases to zero, when the dimensions ratio $d/k$ grows to infinity.  Our next theorem proves this behavior for the special case of $\varphi(t) = t$.  

\begin{figure*}
	\centering
	\begin{subfigure}[t]{0.4\textwidth}
		\includegraphics[scale=0.42]{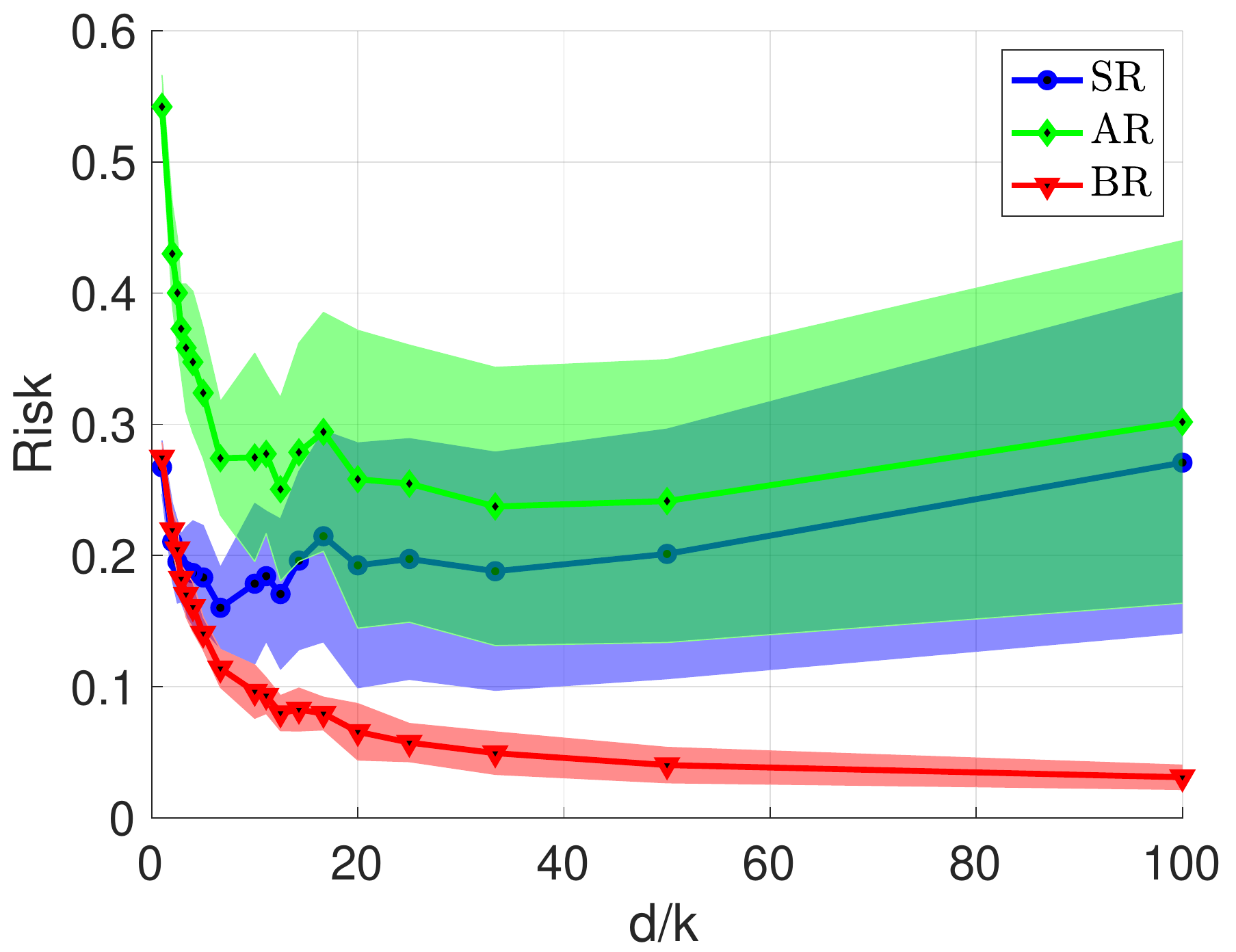}
		\caption{Feature mapping $\vph(t)=t$ and adversary's power $\eps=1$ }
		\label{fig:app:identity-1}
	\end{subfigure}%
	\hfill
	\begin{subfigure}[t]{0.4\textwidth}
		\includegraphics[scale=0.4]{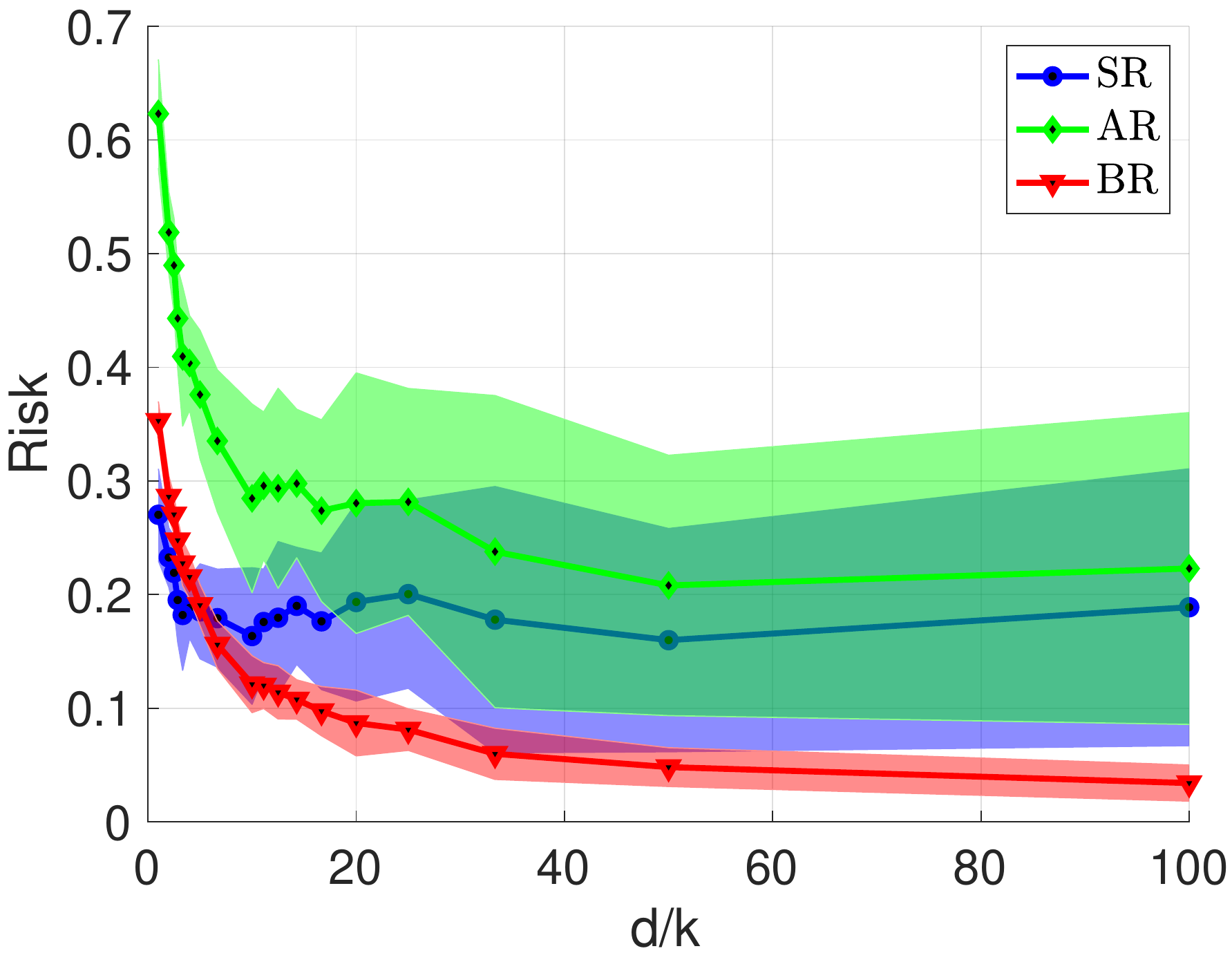}
		\caption{Feature mapping $\vph(t)=t/4+\sign(t)3t/4$ and adversary's power $\eps=1$}
		\label{fig:app:leaky-1}
	\end{subfigure}
		
	\begin{subfigure}[b]{0.4\textwidth}
		\centering
		\includegraphics[scale=0.54]{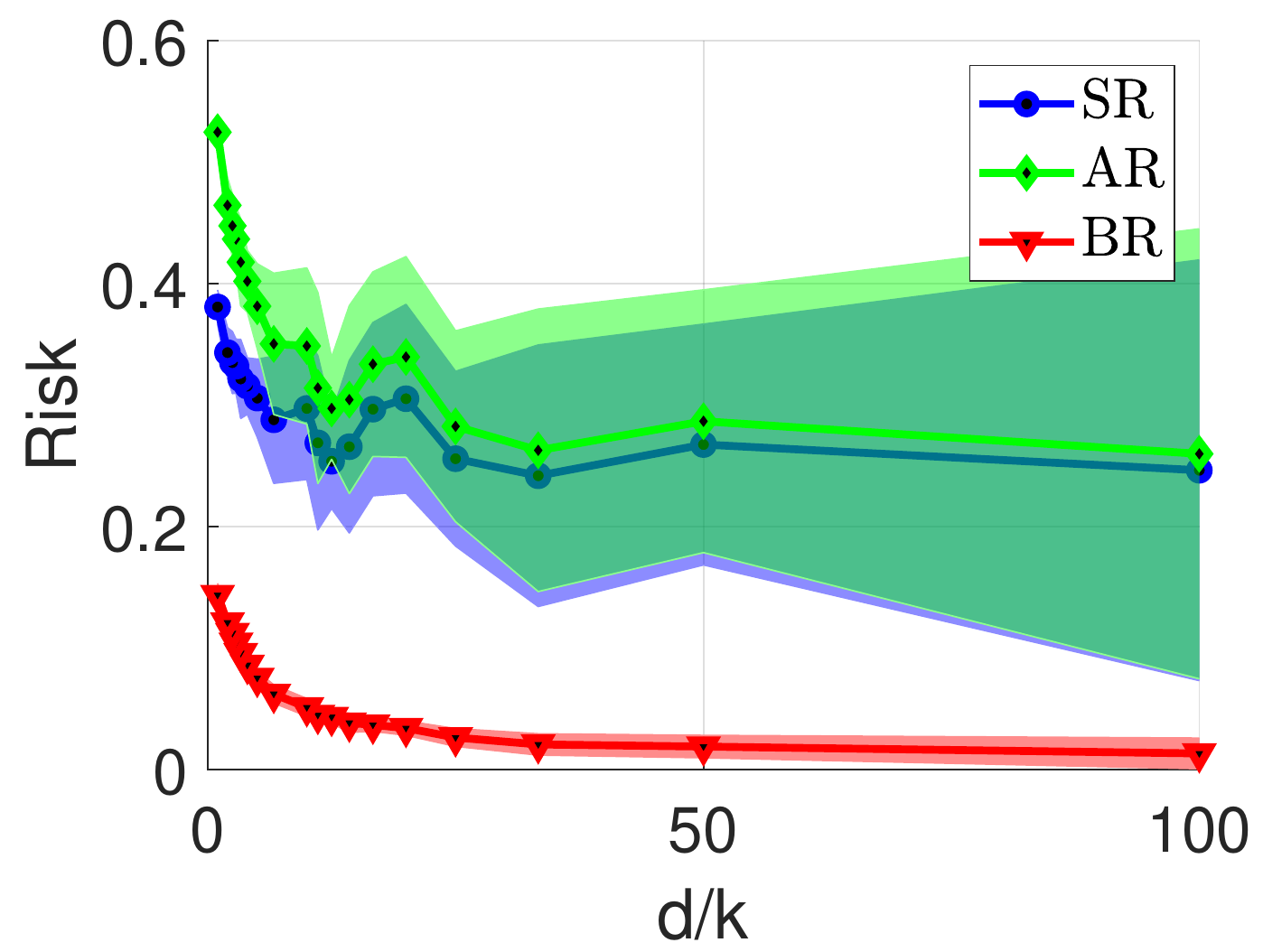}
		\caption{Feature mapping $\vph(t)=t+\sign(t)t^2$ and adversary's power $\eps=1$}
		\label{fig:app:quadratic-1}
	\end{subfigure}%
	\hfill
	\begin{subfigure}[b]{0.4\textwidth}
		\centering
		\includegraphics[scale=0.54]{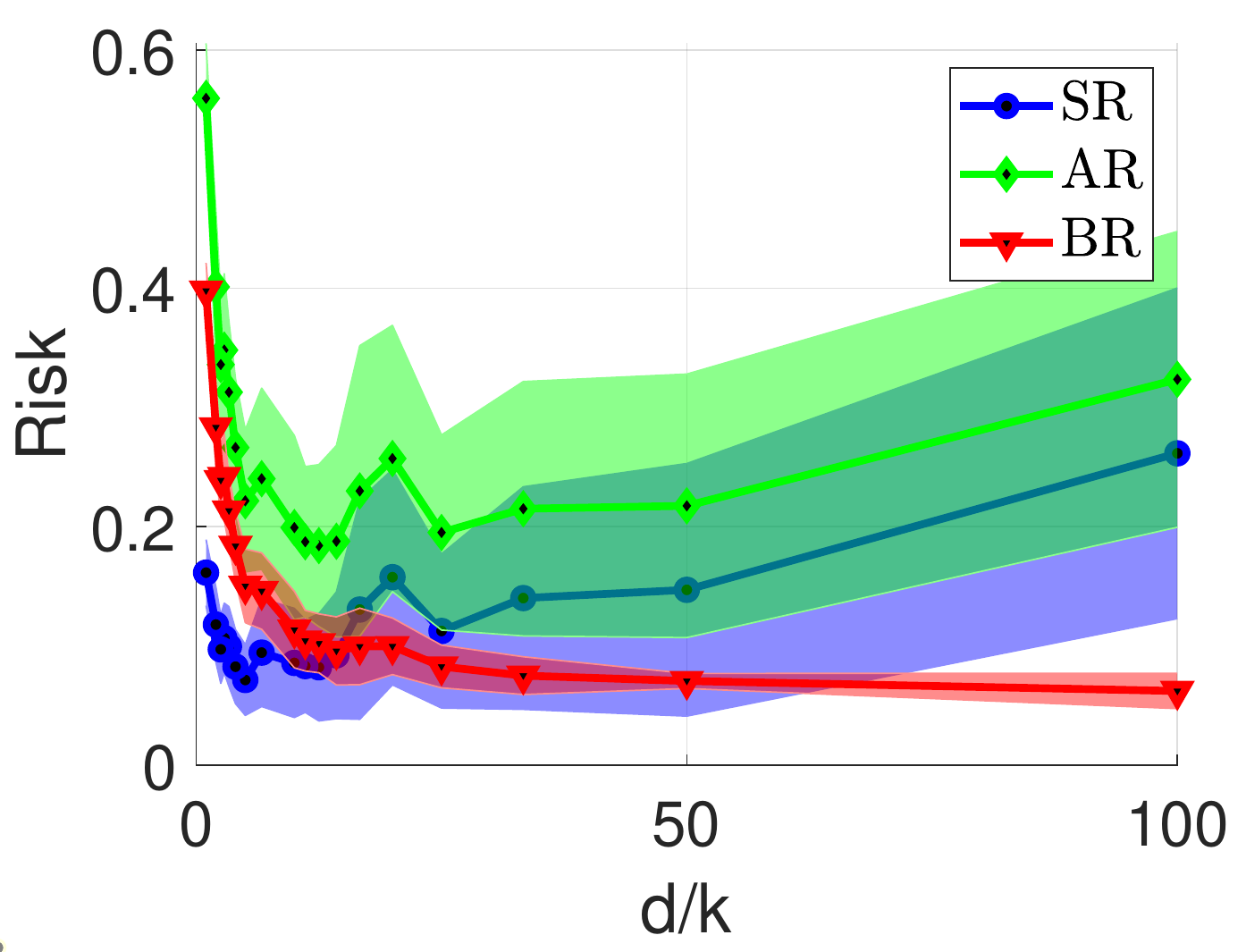}
		\caption{Feature mapping $\vph(t)=\tanh(t)$ and adversary's power $\eps=1$.}
		\label{fig:app:tanh-1}
	\end{subfigure}
	\caption{Effect of dimensions ratio $d/k$ on the standard, adversarial, and boundary risks of the linear classifier $h_\th(x)=\sign(x^\sT \th)$ with $\th$ being the robust empirical risk minimizer~\eqref{eq: tmp9}. 
Samples are generated from the Gaussian mixture model \eqref{eq: gaussian-mix} with balanced classes ($\pi=1/2$), and with four choices of feature mapping $\vph$: (a)\; $\vph(t)=t$, (b)\ $\vph(t)=3t/4+\sign(t) t/4$,
(c)\ $\vph(t)=t+\sign(t)t^2$ and (d)\, $\vph(t)=\tanh(t)$. In these experiments, the ambient dimension $d$ is fixed at $100$, and the manifold dimension $k$ varies from $1$ to $100$. The sample size is $n=300$ the classes average $\mu$ and the weight matrix $W$ have i.i.d. entries from $\normal(0,1/k)$. The adversary's power is fixed at $\eps=1$. For each fixed values of $k$ and $d$, we consider $M=20$ trials of the setup.  Solid curves represent the average results across these trials, and the shaded areas represent one standard deviation above and below the corresponding curves.}
	\label{fig:main}
\end{figure*}


\begin{figure*}
	\centering
	\begin{subfigure}[t]{0.47\textwidth}
		\centering
		\includegraphics[scale=0.4]{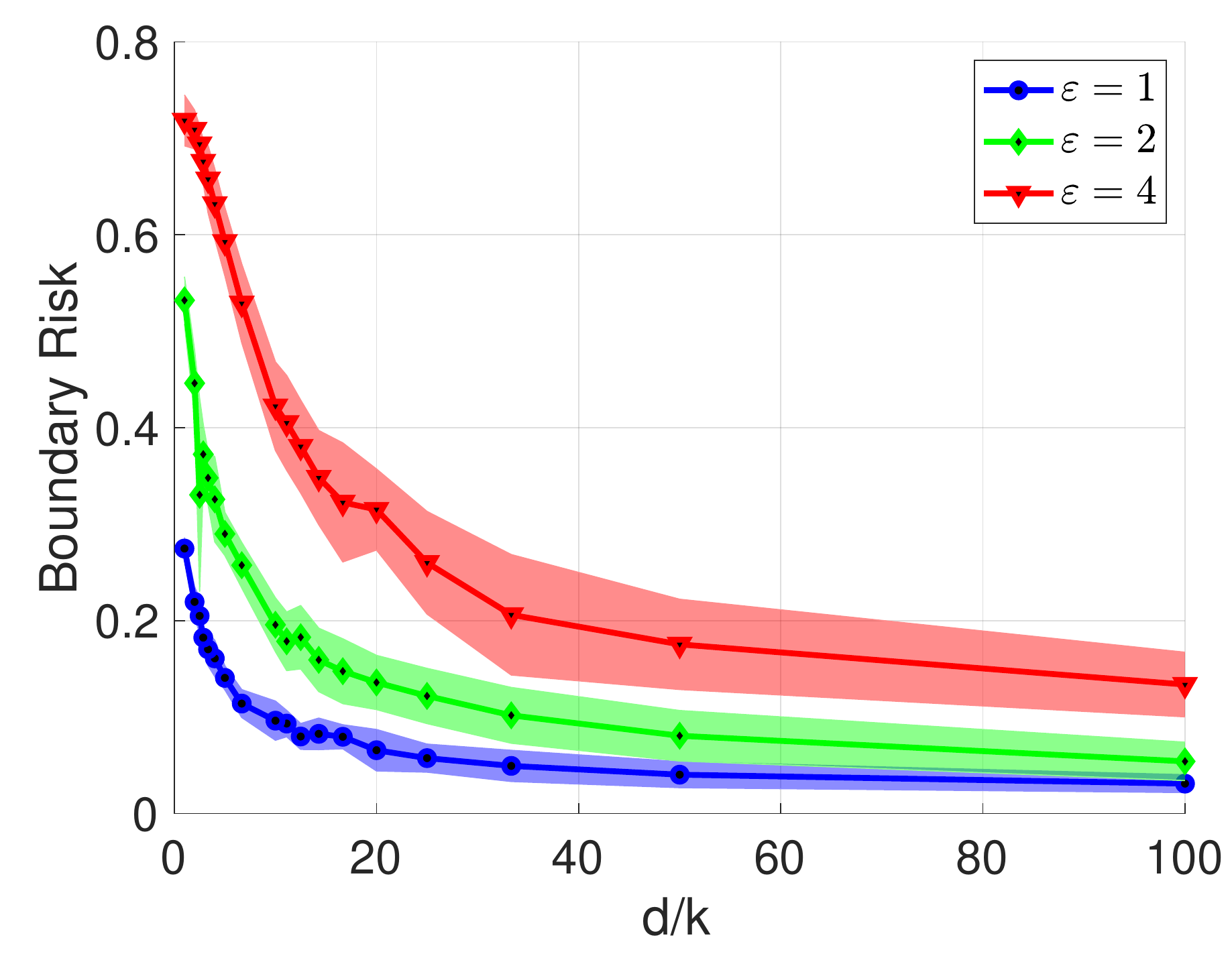}
		\caption{Boundary risk with the feature mapping $\vph(t)=t$ and for multiple values of adversary's power $\eps$ }
		\label{fig:app:identity-2}
	\end{subfigure}%
	\hfill
	\begin{subfigure}[t]{0.47\textwidth}
		\centering
		\includegraphics[scale=0.4]{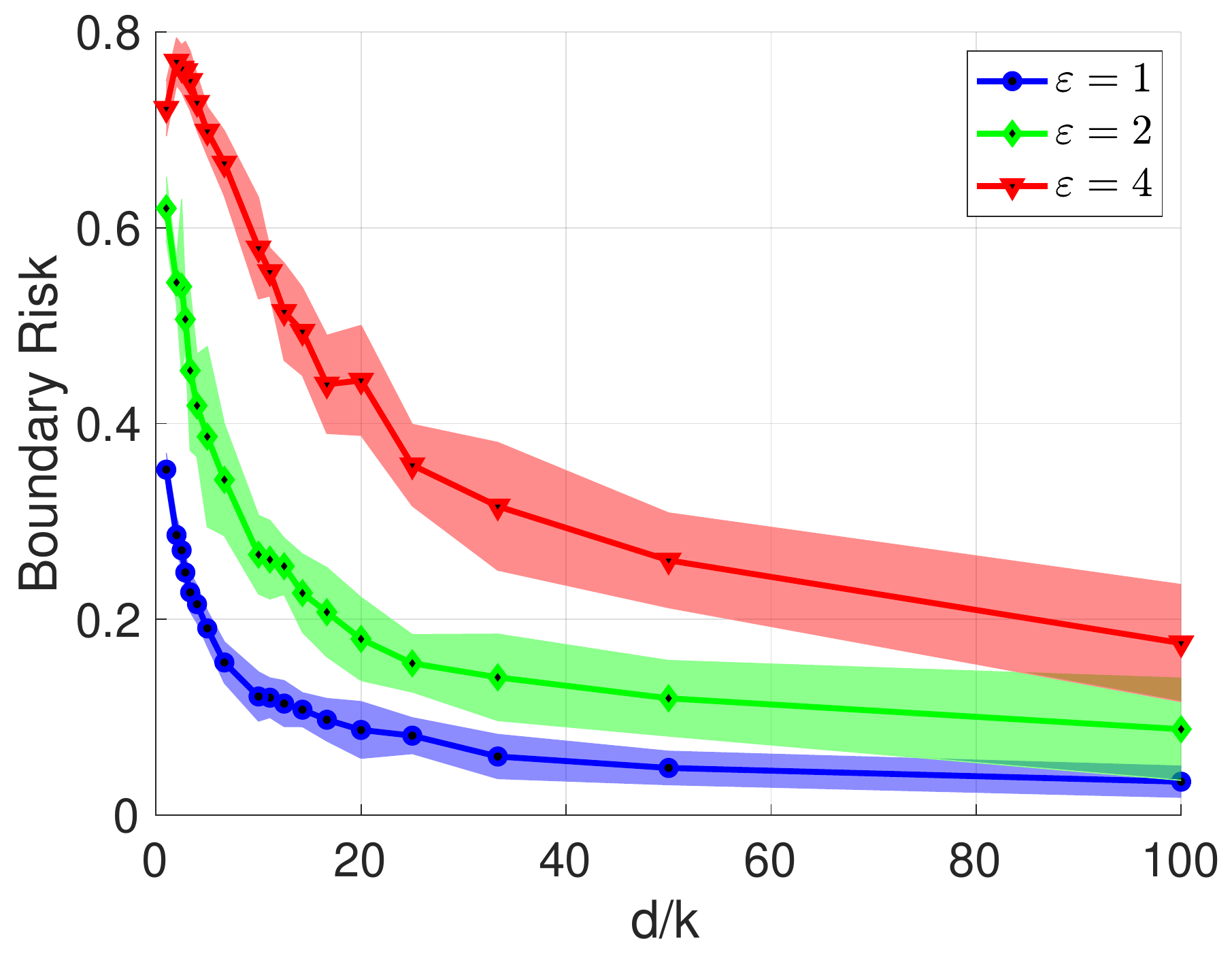}
		\caption{Boundary risk with the feature mapping $\vph(t)=t/4+\sign(t)3t/4$ and for multiple values of adversary's power $\eps$}
		\label{fig:app:leaky-2}
	\end{subfigure}
		
	\begin{subfigure}[t]{0.47\textwidth}
		\centering
		\includegraphics[scale=0.4]{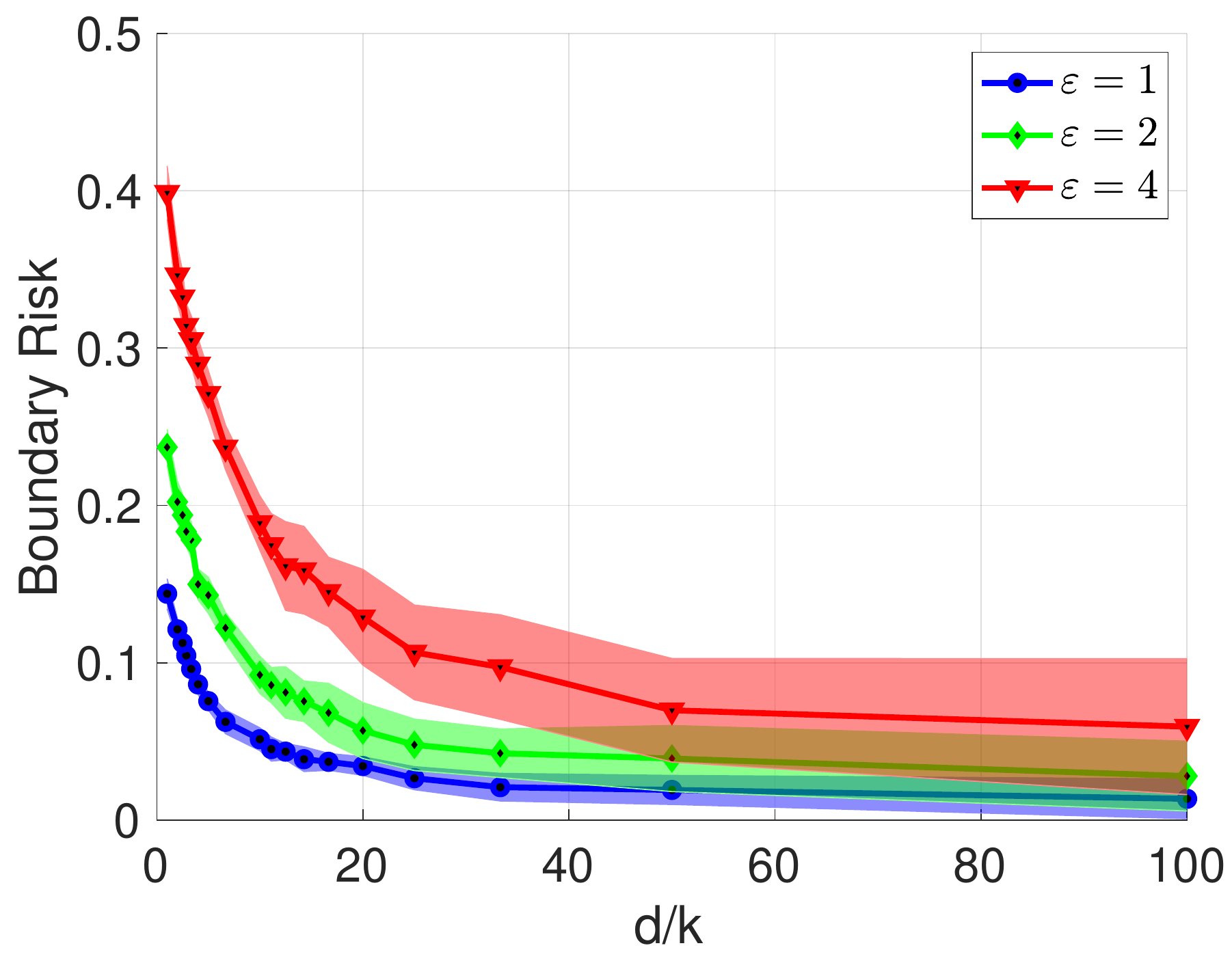}
		\caption{Boundary risk with the feature mapping $\vph(t)=t+\sign(t)t^2$ and for multiple values of adversary's power $\eps$}
		\label{fig:app:quadratic-2}
	\end{subfigure}%
	\hfill
	\begin{subfigure}[t]{0.47\textwidth}
		\centering
		\includegraphics[scale=0.39]{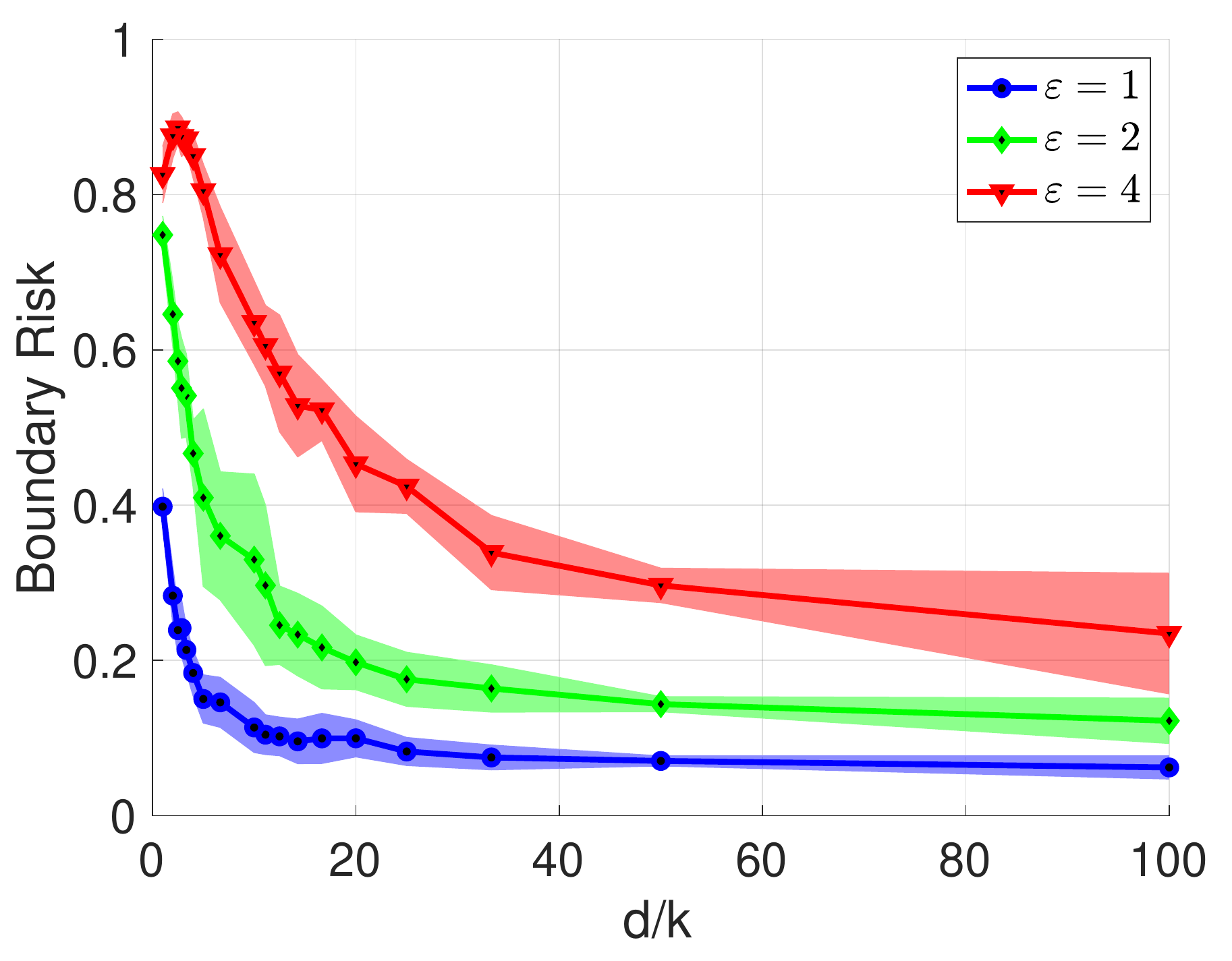}
		\caption{Boundary risk with the feature mapping $\vph(t)=\tanh(t)$ and for multiple values of adversary's power $\eps$}
		\label{fig:app:tanh-2}
	\end{subfigure}
	\caption{Effect of dimensions ratio $d/k$ on the boundary risk of the linear classifier $h_\th(x)=\sign(x^\sT \th)$ with $\th$ being the robust empirical risk minimizer~\eqref{eq: tmp9}. 
Samples are generated from the Gaussian mixture model \eqref{eq: gaussian-mix} with balances classes ($p=1/2$), and with four choices of feature mapping $\vph$: (a)\; $\vph(t)=t$, (b)\ $\vph(t)=3t/4+\sign(t) t/4$,
(c)\ $\vph(t)=t+\sign(t)t^2$ and (d)\, $\vph(t)=\tanh(t)$. In these experiments, the ambient dimension $d$ is fixed at $100$, and the manifold dimension $k$ varies from $1$ to $100$. The sample size is $n=300$ the classes average $\mu$ and the weight matrix $W$ have i.i.d. entries from $\normal(0,1/k)$. We consider different levels of the adversary's power  $\eps\in\{1,2,4\}$. For each fixed values of $k$ and $d$, we consider $M=20$ trials of the setup.  Solid curves represent the average results across these trials, and the shaded areas represent one standard deviation above and below the corresponding curves.}
	\label{fig:app-2}
\end{figure*}

\begin{thm}\label{thm: agnostic}
Consider binary classification under the Gaussian mixture model \eqref{eq: gaussian-mix} with identity mapping $\vph(t)=t$ in the presence of an adversary with $\ell_p$ norm bounded perturbations of size $\eps_p$ for some $p\geq2$. In addition,  Let $h_{\th}(x)=\sign(x^\sT \th)$ be a linear classifier with $\th \in \reals^d$ and assume that as the ambient dimension $d$ grows to infinity, the following condition on the weight matrix $W$ and the decision parameter $\th$ hold:
\begin{equation} \label{eq: agnostic-conditons}
 \frac{\eps_p d^{\frac{1}{2}-\frac{1}{p}}}{\sigma_{\min}(W)} {\left({1-{\twonorm{P_{\mathsf{Ker}(W^\sT)}(\th)}^2}/{\twonorm{\th}^2}}\right)^{-1/2}}=o_d(1)\,,
\end{equation}
where $P_{\mathsf{Ker}(W)}(\th)$ stands for the $\ell_2$-projection of vector $\th$ onto the kernel of the matrix $W$. Then the boundary risk of the classifier $h_\th$ will converge to zero.

In particular, assume that $\hth^\eps$ is the solution of the following adversarial empirical risk minimization (ERM) problem:

\[
\hth^\eps=\arg\min\limits_{\th \in \reals^d}^{} \frac{1}{n}\sum\limits_{i=1}^{n} \sup\limits_{u \in B_{\eps}(x_i)}^{} \ell( y_iu^\sT\th)\,,
\] 
with $\ell:\reals\rightarrow \reals^{\geq 0}$ being a strictly decreasing loss function. In this case, with the weight matrix $W$ satisfying $ \frac{\eps_p d^{\frac{1}{2}-\frac{1}{p}}}{\sigma_{\min}(W)}=o_d(1)$, the boundary risk of $h_{\hth^\eps}$ converges to zero.
\end{thm}

We refer to Section~\ref{proof:thm: agnostic} for the proof of Theorem~\ref{thm: agnostic}.


\section{Boundary risk of Bayes-optimal image classifiers}
We next provide several numerical experiments on the MNIST image data
to corroborate our findings regarding the role of low-dimensional structure of data on the boundary risk of Bayes-optimal classifiers.
Of course, the evaluation of this finding on image data is challenging 
since learning particular structure
of the underlying image distribution is notoriously a difficult problem. There have been a few well-established techniques for this task that we briefly discuss below.

Generative Adversarial Net (GAN) \citep{goodfellow2014generative} is among the most successful methods in modeling the statistical structure of image data. Despite the remarkable success of GANs in generating realistic high resolution images, it has been observed that they may fail in capturing the full data distribution, which is referred to as \textit{model collapse}. In addition, computing the likelihood of image data with GANs requires to perform complex computations. As a direct implication of these observations, it is not statistically accurate and efficient to deploy GANs to formulate the Bayes optimal classifiers \citep{richardson2018gans, richardson2021bayes}.

Fitting elementary statistical models can mitigate the statistical inaccuracies of GANs.  \citep{richardson2018gans} learns the statistical structure of image data by using the class of \textit{Gaussian Mixture Models} (GMM). This choice is motivated by the statistical power of GMMs that they are universal approximators of probability densities \citep{goodfellow2016deep}. On the other hand, working with general Gaussian covariance matrices can make the estimation problem both in terms of computational cost and memory storage extremely prohibitive. \citep{richardson2018gans} deployed \textit{Mixture of Factor Analyzers (MFA)} \citep{ghahramani1996algorithm} to avoid storage and computation with such high-dimensional matrices.  This deployment is aligned with the former intuition that the space of meaningful images is indeed a small portion of the entire high-dimensional space. In addition, they show that with moderate number of components in the GMM one can produce adequately realistic images, and further reduce the computational burden. 

We will adopt the MFA procedure introduced in \citep{richardson2018gans} to generate realistic image data for our numerical experiments. The main reasons for this adoption are: i) this framework is flexible for generating realistic images from a low-dimensional subspace, and ii) it enables us to accurately and efficiently calculate the log-likelihood of images, which can be used later to formulate the Bayes-optimal classifier. It is worth noting that, using a class of less complex models, in this case GMMs rather than GANs, will output images with lower resolution, which is not a major concern for the main purpose of this numerical study. In the next sections, we first provide a brief overview of the GMM estimation steps and then review some of the standard frameworks to produce adversarially crafted examples.

\subsection{Learning image data with Gaussian Mixture Models (GMM)}
A general setup for fitting a GMM to image data $\{x_i\}_{1:n}\in \reals^d$ is based on the following model
\[
x\sim\sum_{k=1}^{K}\alpha_k\normal(\mu_k,\Sigma_k)\,, \quad \Sigma_k\in \reals^{d\times d}\,,~~ \mu_k\in \reals^d\,,
\]
where $K$ denotes the number of components in GMM and $\alpha_i$ are mixing weights.  This problem, without imposing any extra structure on image data,  involves learning $O(Kd^2)$ parameters which can be extremely difficult for high-dimensional images. \citep{richardson2018gans} deployed a mixture of factor analyzers (MFA), where they use tall matrices $A_{d\times \ell}$ to embed a low-dimensional subspace in the full data space. In this case, the following model is considered

\[
x\sim\sum_{k=1}^{K}\alpha_k\normal(\mu_k,A_kA_k^\sT+D_k)\,, \quad  A_k\in \reals^{d\times \ell},~~D_k\in \reals^{d\times d}\,,~~ \mu_k\in \reals^d\,,
\]
where $D_k$ is a diagonal matrix showing the variance on each single pixel. This model ameliorates the previous high storage and computational cost, as in this case $Kd(\ell+2)$ learning parameters exist, which scales linearly with image dimension $d$. This model is intimately related to the specific case of the low-dimensional manifold models on features in \eqref{eq: gaussian-mix} and \eqref{eq: generative} with the identity mapping $\varphi(t)=t$. The only difference is in the entrywise independent Gaussian noise coming from diagonal matrix $D_k$, however in the numerical experiments we observed that indeed the estimated values of $D_k$ entries are extremely small, which makes this difference negligible. We use the maximum likelihood estimator to compute the model parameters.   For this purpose, we need to compute the log-probabilities.  For a single component of GMM we first compute the likelihood $p(x|A,D,\mu)$ given by
\begin{equation}\label{eq: likelihood}
p(x|A,D,\mu)=-\frac{1}{2}\left[ d \log (2\pi)+\log \det (AA^\sT+D)+(x-\mu)^\sT\big(D+AA^\sT\big)^{-1}(x-\mu) \right]\,.
\end{equation}

\citep{richardson2018gans} used the following algebraic identities with $u=x-\mu$, $L=A^\sT D A\in \reals^{\ell \times \ell}$ to avoid large matrix storage and multiplications:
 \begin{align*}
     &u^\sT \Sigma^{-1}u=u^\sT\left[ D^{-1}u -D^{-1}A L^{-1} (A^\sT D^{-1}u) \right]\,,\\
     &\log \det (AA^\sT+D)=\log \det L +\sum\limits_{j=1}^{d}\log(d_j) \,,
 \end{align*}
 where $d_i$ denotes the $i$-th entry on the diagonal of matrix $D$.  In addition, \citep{richardson2018gans} employed \textit{differentiable programming} framework to efficiently solve the corresponding Maximum-likelihood optimization problem on GPU. We use their publicly available code at \url{https://github.com/eitanrich/gans-n-gmms} to fit GMMs on full image data sets. As a simple example, we consider models with $K=10$ components and the manifold dimensions $\ell=1,10,100$. We fit these three models to the training samples of the MNIST data set \citep{deng2012mnist} that are labeled ``6".  Figure \ref{fig: sample-six} exhibits sample images generated from the learned GMM models.

\begin{figure*}
	\centering
	\begin{subfigure}[t]{0.31\textwidth}
		\centering
		\includegraphics[scale=0.45]{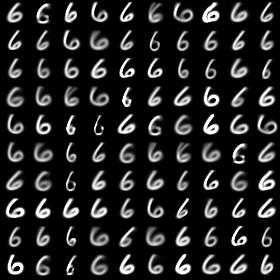}
		\end{subfigure}
	\hfill
	\begin{subfigure}[t]{0.31\textwidth}
		\centering
		\includegraphics[scale=0.45]{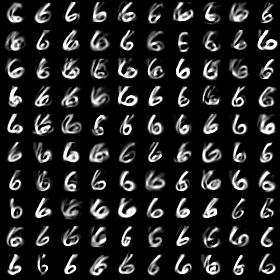}
		\end{subfigure}
         \hfill
\begin{subfigure}[t]{0.31\textwidth}
		\centering
		\includegraphics[scale=0.45]{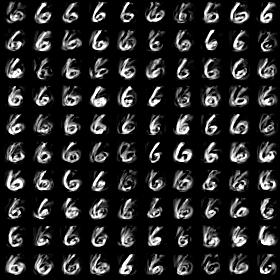}
		\end{subfigure}
	 \caption{Sample generated images from a GMM model fit on MNIST training set images with label six. Three GMM models with number of components $K=10$, and manifold dimensions $\ell=1,10,100$ (from left to right) are considered. }\label{fig: sample-six}
\end{figure*}

\subsection{PGD and FGM adversarial attacks}

Recall that adversarial examples are meant to be close enough to original samples, yet be able to degrade the classifier performance. For a loss function $\cL$ consider the adversarial optimization problem 
\[
\max \limits_{\|x-x'\|_{\ell_2}\leq \eps }^{} \cL(h(x'),y)\,.
\]
Using the 0-1 loss $\cL(h(x'),y)=\ind(y\neq h(x'))$ yields the inner optimization problem \eqref{eq: AR}. A large body of proposed methods to produce adversarial examples consider the first-order linear approximation of the loss function around the original sample. More precisely, $x'$ is written as $x+\delta$, and then single/multi steps of gradient descent (GD) of the negative loss function is considered. In this framework, first a powerful predictive model, e.g. a neural network, is fit to the training samples, which will be used as a surrogate for the learner's model $h$. For $\ell_2$-bounded adversary's budget $\eps$, the \textit{Fast Gradient Method} (FGM) performs a single step of normalized GD which yields
\[
x'=x+\eps \frac{\nabla_x \cL(h(x),y)}{\| \nabla_x \cL(h(x),y) \|_{\ell_2}}\,.
\]
This method is first introduced in \citep{goodfellow2014explaining} for $\ell_{\infty}$-bounded attacks under the name \textit{Fast Gradient Sign Method (FGSM)}. Other variants for other $\ell_p$-bounded adversarial attacks are introduced in \citep{tramer2017space}, generally called the FGM (removing ``sign" from FGSM).  A more general scheme to produce adversarial examples is via a multi-step implementation of the above procedure with the projected gradient descent (PGD). This attack is introduced in  \citep{madry2017towards}, with iterative updates  given by
\[
x^{t+1}=\Pi_{B_\eps(x)}\left(x^t+\eps \frac{\nabla_x \cL(h(x),y)}{\| \nabla_x \cL(h(x),y) \|_{\ell_2}} \right)\,,
\]
where $\Pi_{B_\eps(x)}$ stands for the projection operator to the $\ell_2$-ball centered at $x$ with radius $\eps$. For our image classification numerical experiments, we will use FGM and PGD attacks to produce adversarial examples . We follow the same implementation of PGD and FGM adversarial attacks provided in CleverHans library v4.0.0 \citep{papernot2016technical}. The code for this implementation can be accessed at \href{https://github.com/cleverhans-lab/cleverhans}{https://github.com/cleverhans-lab/cleverhans}. In our implementation, the original image values are normalized to be in the interval $[0,1]$, and we clip perturbed pixel values to be in the same interval. We next present key findings from our numerical experiments.

\subsection{Main experiments and key findings}
In this section, we connect the previously described parts. Put all together, these are the three main steps of our experiments: 
\begin{enumerate}
    \item For several choices of $K$ (number of components) and latent dimensions $\ell$, we first fit two GMM models to zeros and sixes of the training set of MNIST data set. By deploying the learned models, we generate $5000$ new images with uniform probability on labels 6 or 0, i.e. at the beginning of generating an image, with equal chance we decide to use either of the models.  In addition, for the defined binary classification problem (0 vs 6), we deploy \eqref{eq: likelihood} to obtain two likelihood models $p_0(x),p_1(x)$.  This can be used to formulate the Bayes-optimal classifier $\ind(p_1(x)>p_0(x))$. 
    \item In this step, we adversarially attack the generated images. To this end, the data set is split into 80$\%$-20$\%$ training-test samples. The training set is used to train a neural network for PGD and FGM attacks.  The obtained model will be used later to craft adversarial examples for the $1000$ test images.
    \item Finally, the performance of the Bayes-optimal classifier on adversarially perturbed test images (size $1000$) is evaluated. 
\end{enumerate}

In our first experiment, we consider a fixed number of components $K=10$ along with three different latent dimensions $\ell=1,10,$ and $100$. For each pair of $(K,\ell)$, we randomly select one sample with label $6$ among the original $1000$ test images. For PGD adversarial attacks, we start from $\eps=0$, and incrementally increase the adversary's $\ell_2$-bounded power $\eps$ until to the point that the Bayes-optimal classifier fails to correctly label the sample. Figure \ref{fig: adv-3 by 3} displays the original, adversarial attack, and the perturbed images for each $
\ell$ value. The Bayes-optimal classifier fails at adversarial power $\eps=22.6,11,7$ for $\ell=1,10,100$ respectively. The result conforms to the fact that samples coming from a higher value of dimension ratio $d/\ell$ ( here $d=784$) indeed require stronger adversarial attacks (larger $\eps$).

\begin{figure}
    \centering
    \includegraphics[scale=1.5]{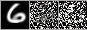}\\
    \centering
    \includegraphics[scale=1.5]{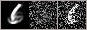}\\
    \centering
     \includegraphics[scale=1.5]{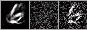}
    \caption{Sample attacks and perturbed images for three different latent dimensions $\ell=1,10,100$, respectively from top to bottom. In each row, from left to right, the original sample, the adversarially crafted perturbation, and the perturbed image is exhibited. The original images are generated from a GMM with the number of components $K=10$. In all images under the PGD adversarial perturbation, we start with the $\ell_2-$bounded adversarial power $\eps=0$, and incrementally increase it to the point that the Bayes-optimal classifier fails to infer the correct label. In this experiment, the Bayes-optimal misclassification on sampled images with $\ell=1,10,100$ occurs at $\ell_2-$bounded adversarial power $\eps=22.6, 11, 7$, respectively. 
    It can be seen that samples coming from smaller latent dimension $\ell$ require stronger perturbation to get misclassified by the Bayes-optimal classifier.}
    \label{fig: adv-3 by 3}
\end{figure}

In the second experiment, we consider the two choices of $K=1$ and $K=10$ and vary the latent dimension $\ell$. For each pair of $(K,\ell)$, we repeat the above three-step procedure for adversary's $\ell_2$-bounded power $\eps=12$ and compute the boundary risk of the Bayes-optimal classifier on the PGD and FGM perturbed images.  
The plots are included in Figure \ref{fig: BR-real image}.
As observed, by increasing the dimension ratio $d/\ell$, the boundary risk of the Bayes-optimal classifier decreases to zero.

\begin{figure*}
	\centering
	\begin{subfigure}[t]{0.45\textwidth}
		\centering
		\includegraphics[scale=0.65]{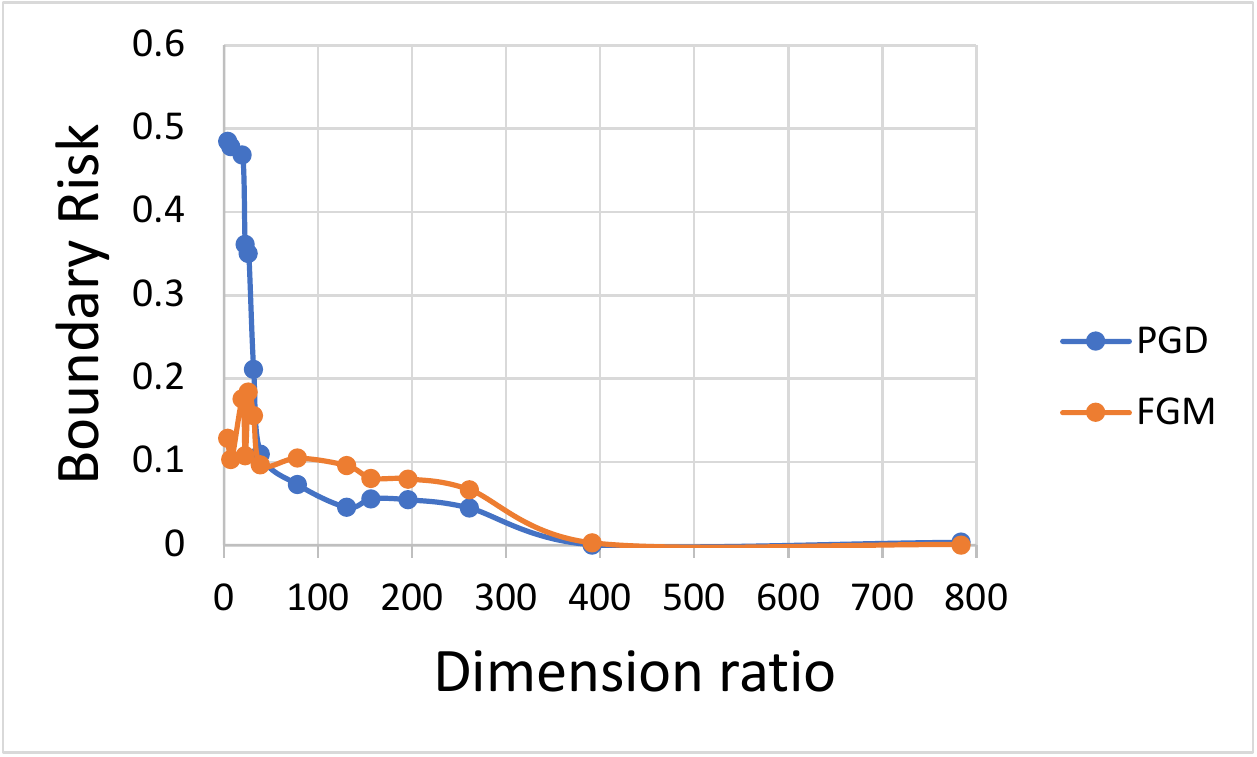}
		\end{subfigure}
	\hfill
	\begin{subfigure}[t]{0.45\textwidth}
		\centering
		\includegraphics[scale=0.65]{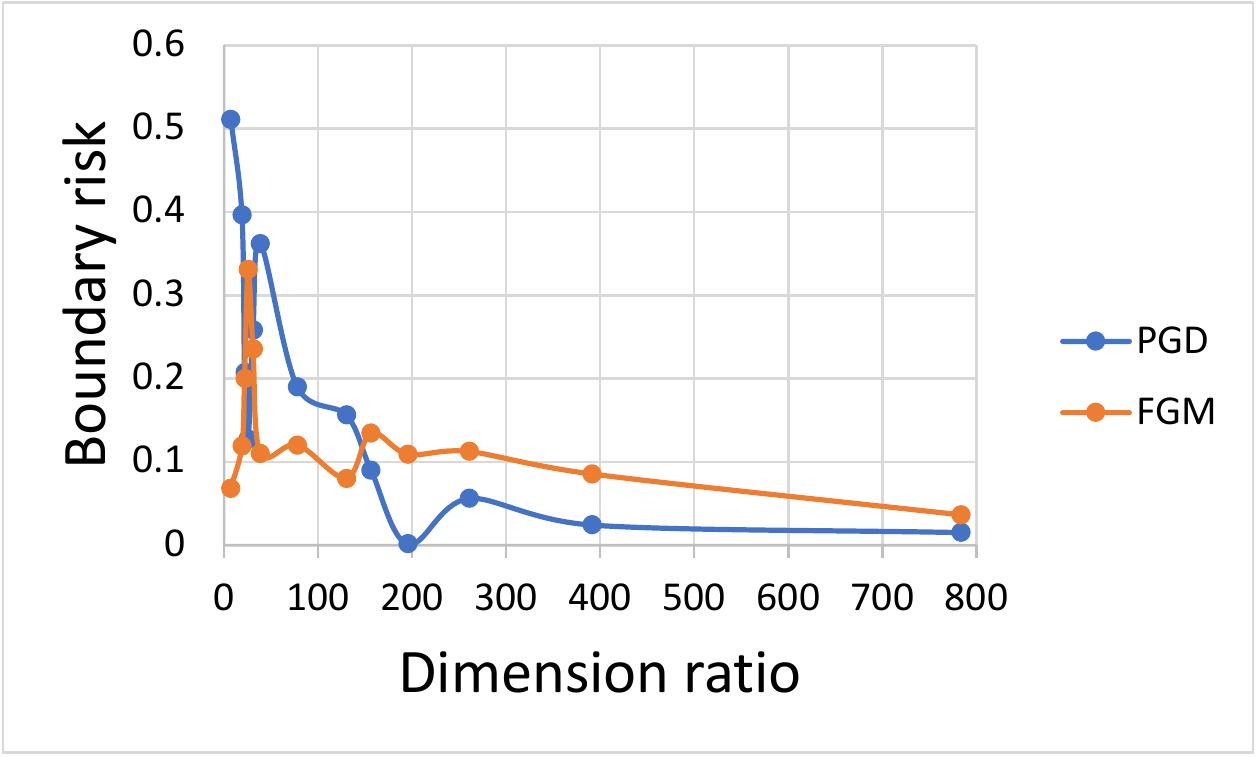}
		\end{subfigure}
	 \caption{Boundary risk of the Bayes-optimal classifier on $1000$ test images generated from GMM models with number of components $K=1$ and $K=10$ (from left to right). In each experiment, the adversary's $\ell_2$-bounded perturbation power is fixed at $\eps=12$, and both adversarial attacks PGD and FGM are considered. It can be observed that the boundary risk of the Bayes-optimal classifier will converge to zero as the dimension ratio $d/\ell$ grows.}\label{fig: BR-real image}
\end{figure*}


  

\section{Conclusion}
In this paper, we studied the role of data distribution (in particular latent low-dimensional manifold structures of data) on the tradeoff between robustness (against adversarial perturbations in the input, at test time) and generalization (performance on test data drawn from the same distribution as training data). We developed a theory for two widely used classification setups (Gaussian-mixture model and generalized linear model), showing that as the ratio of the ambient dimension to the manifold dimension grows, one can obtain models which both are robust and generalize well. This highlights the role of exploiting underlying data structures in improving robustness and also in mitigating the tradeoff between generalization and robustness. Through numerical experiments, we demonstrate that low-dimensional manifold structure of data, even if not exploited by the training method, can still weaken the robustness-generalization tradeoff.

\clearpage
\appendix

\input{supplement2}


\section*{Acknowledgement}
A.~Javanmard is partially supported by the Sloan Research Fellowship
in mathematics, an Adobe Data Science Faculty Research Award and the
NSF CAREER Award DMS-1844481.
\bibliographystyle{apalike}
\bibliography{mybib,Bibfiles2,mm_bib}





\end{document}

%% file: notation.tex
\newcommand{\twonorm}[1]{\left\|#1\right\|_{\ell_2}}
\newcommand{\fronorm}[1]{\left\|#1\right\|_{F}}

\def\tx{\widetilde{x}}

\def\reals{{\mathbb R}}

\def\eps{{\varepsilon}}
\def\prob{{\mathbb P}}
\def\E{{\mathbb E}}

\def\L0{{L_i}}

\def\<{\langle}
\def\>{\rangle}

\def\hth{\widehat{\theta}}

\def\F{{\sf F}}
\def\ind{{\mathbb I}}

\def\F{{\sf F}}
\def\normal{{\sf N}}

\def\sT{{\sf T}}

\def\sign{{\rm sign}}

\def\v*{v_i}
\def\T*{T_i}

\def\u*{u_i}
\def\F*{F_i}

\def\tW{\widetilde{W}}

\def\tx{\tilde{x}}

\def\cL{\mathcal{L}}

\def\th{\theta}
\def\hth{{\widehat{\theta}}}

\def\cX{\mathcal{X}}

\newcommand{\pnorm}[1]{\left\|#1\right\|_{\ell_p}}

\def\l1u{W}

\newcommand{\ajcomment}[1]{}

\makeatletter
\newcommand{\labitem}[2]{%
\def\@itemlabel{\text{#1}}
\item
\def\@currentlabel{#1}\label{#2}}
\makeatother

\DeclareMathAlphabet{\mathpzc}{OT1}{pzc}{m}{it}


%% file: supplement2.tex
\section{Proof of theorems and technical lemmas}\label{sec:proofs}

\subsection{Proof of Proposition \ref{propo: bayes}} \label{proof: propo: bayes}
Denote the Bayes-optimal classifier by $h^*(x):=\sign(\eta(x)-1/2)$. Consider a classifier $h$ where the set $A=\{x: h(x)=+1\}$ is Borel measurable. We denote the complement of $A$ with $A^c=\reals^d\backslash A$.  Our goal is to show that $\SR(h^*)\leq\SR(h)$. First, for every $x\in \reals^d$, and a Borel measurable set $B$, it is easy to check that the following holds:
\begin{equation}\label{eq: tmp10}
(2\eta(x)-1)\cdot \left(  \ind(x\in B)-\ind(\eta(x)\leq 1/2)\right)  \geq 0\,.
\end{equation}
On the other hand, we have
\begin{align*}
\SR(h)&=\prob(h(x)y\leq0)\\
&=\E[ \ind(x\in A^c,y=+1)+\ind(x \in A,y=-1)]\\
&=\E_x\left[\E[\ind(x\in A^c,y=+1)+\ind(x\in A,y=-1)|x] \right]\\
&=\E_x\left[\ind(x\in A^c) \eta(x) + \ind(x\in A)(1-\eta(x))\right]\,.
\end{align*}
By using the identity $\ind(x\in A)+\ind(x \in A^c)=1$, we can simplify the above equation to arrive at the following:
\begin{align}\label{eq: tmp10-2}
\SR(h)&=\E_x\left[\ind(x\in A^c)(2\eta(x)-1) +1-\eta(x)\right]\,.
\end{align}
Note that the above equation holds for all classifiers $h$. In particular, we can employ it for $h^*$. For this purpose, we know that the subset of $\reals^d$ with the negative label is ${A^*}^{c}=\{x:\eta(x)\leq 1/2\}$. By recalling  \eqref{eq: tmp10-2} for $h^*$ we get: 
\begin{align}\label{eq: tmp10-3}
\SR(h^*)=\E_x\left[\ind(\eta(x)\leq 1/2) (2\eta(x)-1) + 1-\eta(x) \right]\,.
\end{align}
We can subtract  \eqref{eq: tmp10-3} from \eqref{eq: tmp10-2}, and then exploit \eqref{eq: tmp10} with substituting $B$ with $A^c$ to get $\SR(h)\geq\SR(h^*)$. This completes the proof.

\subsection{Proof of Corollary \ref{coro: optimals}}
In order to characterize the Bayes-optimal classifier of the Gaussian mixture setting \eqref{eq: gaussian-mix}, we first need to compute the conditional density function $\prob(y=+1|x)$. This will help us use Proposition \ref{propo: bayes} to identify the Bayes-optimal classifiers. For this purpose, in the first step, consider the general Gaussian mixture model  $\tx\sim \normal(y\tmu, \Sigma)$, where the covariance matrix $\Sigma$ is not necessarily full-rank. This means that $\tx$ can be a degenerate multivariate Gaussian with the following density function:
\begin{equation}\label{eq: tmp13-1}
f_{\tx|y}(\tx)=\left(|2\pi \Sigma|_+\right)^{-1/2}\exp\left(-(\tx-y\tmu)^\sT\Sigma^\dagger (\tx-y\tmu)/2\right)\,,
\end{equation}
where $|.|_+$ stands for the pseudo-determinant operator. By recalling the Bayes' theorem we get

\begin{align}\label{eq: tmp13}
\prob(y=+1|\tx)&=\frac{\prob(y=+1)f_{\tx|y=+1}(\tx) }{\prob(y=+1)f_{\tx|y}(\tx)+\prob(y=-1)f_{\tx|y=-1}(\tx)  }\nonumber\\
&=\frac{1}{1 +\left((1-\pi)/\pi\right)\cdot f_{\tx|y=-1}(\tx)/ f_{\tx|y=+1}(\tx)  }\nonumber\,.
\end{align}
By using \eqref{eq: tmp13-1} in the last equation, we will arrive at the following
\begin{equation}\label{eq: tmp13}
\prob(y=+1|\tx)=\left(1+\exp(-2\tx^\sT\Sigma^\dagger\tmu+q)\right)^{-1}\,,
\end{equation}
where $q=\log \frac{1-\pi}{\pi}$. On the other hand, it is easy to observe that 
\begin{equation}\label{eq: tmp12}
\sign((1+\exp(-t))^{-1}-1/2)=\sign(t) \,.
\end{equation}
Finally, we can deploy \eqref{eq: tmp13} in Proposition \ref{propo: bayes} in conjunction with the identity \eqref{eq: tmp12} to derive the Bayes-optimal classifier $\widetilde{h}^*=\sign(\tx^\sT\Sigma^\dagger\tmu -q/2)$. We can now focus on the primary setup \eqref{eq: gaussian-mix}. Let $\Sigma:=WW^\sT$, $\tx=\vph^{-1}(x)$, and $\tmu:=W\mu$. It is easy to check that with these new notations we have $\tx\sim\normal(y\tmu, \Sigma)$. Recall the Bayes-optimal classifier of this setting  $h^*(x)=\sign(\tx^\sT\Sigma^\dagger\tmu -q/2)$. By replacing $\tx$, $\Sigma$, and $\tmu$ by their respective definitions $\vph^{-1}(x)$, $WW^\sT$, and $W\mu$ we realize that the Bayes-optimal classifier is given by
$$h^*(x)= \sign\left(\vph^{-1}(x)^{\sT} (WW^{\sT})^{\dagger} W\mu-q/2\right)\,.$$ 

In this part we want to characterize the Bayes-optimal classifier of the generalized linear model \eqref{eq: generative}. To this end, note that if the proposed classifier in Corollary \ref{coro: optimals} is optimal for the identity map $\vph(x)=x$, then we can simply consider $\tx=\vph^{-1}(x)$, and establish its optimality for every function $\vph$. This means that we can only focus on the case with the identity function $\vph(x)=x$. For the purpose of identifying the Bayes-optimal classifiers, we can use Proposition \ref{propo: bayes}. In the first step, we need to compute the conditional probability $\prob(y=1|x)$. By using the Bayes rule we get

\begin{align*}
\prob(y=+1|X=x)&=
\frac{\int\prob_{X|Z}(x|z)\prob_{Y|X,Z}(+1|x,z) d\prob_Z(z)}{\int\prob_{X|Z}(x|z)d\prob_Z(z)}\\
&=\frac{\int\ind\left(Wz=x\right)\prob_{Y|Z}(+1|z) d\prob_Z(z)}{\int\ind\left(Wz=x\right)d\prob_Z(z)}\,,
\end{align*}
where in the last equation, we used the fact that condition on $z$, the feature vector $x$ and the label $y$ are independent. In addition, as $W$ is a full-rank matrix (linearly independent columns), then for a fixed $x$, equation $Wz=x$ has the unique solution $z^*=(W^\sT W)^{-1}W^{\sT}x$. This gives us
\begin{align*}
\prob(y=+1|X=x)&=\prob_{Y|Z}(+1|z^*)\\
&=f(\beta^\sT (W^\sT W)^{-1}W^{\sT}x)\,.
\end{align*}    
By recalling Proposition \ref{propo: bayes}, we realize that the Bayes-optimal classifier of this setting with $\vph(t)=t$ is given by  

$$h^*(x)=\sign\left(f(\beta^\sT (W^\sT W)^{-1}W^{\sT}x)-1/2\right)\,.$$
This completes the proof.

\subsection{Proof of Theorem \ref{thm: Gaussian-mix}}\label{proof: thm: Gaussian-mix}
We first show that if the result holds for $\ell_2$ adversaries, then the theorem is also true for $\ell_p$ norm bounded adversaries with power $\eps_p$. Indeed this result can be seen as an immediate consequence of  
\begin{equation}\label{eq: tmp-holder}
\{x': \pnorm{x'-x}\leq \eps_p\} \subseteq B_\eps(x)\,,
\end{equation}
where $\eps=\eps_p d^{\frac{1}{2}-\frac{1}{p}}$. More precisely, the boundary risk of $\ell_p$ adversaries with power $\eps_p$ will be smaller than $\ell_2$ adversaries with power $\eps_p d^{\frac{1}{2}-\frac{1}{p}}$ which completes the proof. In this step, we only need to prove \eqref{eq: tmp-holder}. For this end,  from the Holder's inequality for every $u,v\in \reals^d$ and $r\geq1$ we have:

\[
\sum_{i=1}^d |u_i||v_i| \leq \left(\sum_{i=1}^{d}|u_i|^{r}\right)^{\frac{1}{r}} \left(\sum_{i=1}^{d}|v_i|^{\frac{r}{r-1}}\right)^{1-\frac{1}{r}}\,.
\]  
By using $u_i=|x'_i-x_i|^2$, $v_i=1$, and $r=p/2$ (note that $p\geq 2$) in the above inequality we get $\twonorm{x'-x}\leq \pnorm{x'-x}d^{{1}/{2}-{1}/{p}}$, which yields \eqref{eq: tmp-holder}.

We next focus on showing the result for $\ell_2$ norm bounded adversaries with power $\eps$. For this end, introduce $\rho(.)=\vph^{-1}(.)$, $\Sigma=WW^\sT$, $\tx=\rho(x)$,  and $\tmu=W\mu$. From Corollary \ref{coro: optimals}, we know that the Bayes-optimal classifier is given by 
\begin{equation} \label{eq: tmp3-1}
h^*(x)= \sign\left(\rho(x)^{\sT} (WW^{\sT})^{\dagger} W\mu-q/2\right)\,.
\end{equation}
We next focus on computing the boundary risk of the Bayes-optimal classifier $h^*$. By recalling the boundary risk definition we get:
\begin{align*}
\BR(h^*)&= \prob_{x,y}\left( h^*(x)y\geq 0,  \inf\limits_{u \in B_\eps(x)}^{} h^*(x)h^*(u) \leq 0 \right)\,.
\end{align*}
In the next step, we expand the above expression for the two possible values of $y\in \{+1,-1\}$ to get 
\begin{align*}
\BR(h^*)&=\prob_{x,y}\left( h^*(x)\geq 0,  \inf\limits_{u \in B_\eps(x)}^{}h^*(u) \leq 0 , y=1\right)\\
&\;+\prob_{x,y}\left( h^*(x)\leq 0,  \sup\limits_{u\in B_\eps(x)}^{} h^*(u) \geq 0 , y=-1\right) \nonumber\,.
\end{align*}
We then plug \eqref{eq: tmp3-1} into the last equation to get 
\begin{align*}
\BR(h^*)&\leq \prob_{x,y} \big( \inf\limits_{ u\in B_\eps(x)}\tmu^{\sT} \Sigma^{\dagger} \rho(u) \leq \tfrac{q}{2}\leq  \tx^\sT \Sigma^{\dagger}\tmu,\; y=1\big)\\
&\;+\prob_{x,y} \big(  \tx^\sT \Sigma^{\dagger}\tmu \leq \tfrac{q}{2}\leq \sup\limits_{u\in B_\eps(x)} \tmu^\sT \Sigma^{\dagger}\rho(u),\; y=-1\big) \,.
\end{align*}
By our assumption, 
$d\vph/dt\geq c$, for some constant $c>0$.  By simple algebraic manipulation it is easy to check that 
$\rho\left(B_\eps(x)\right)$ is a subset of $B_{\eps/c}(\rho(x))$. This gives us
\begin{align*}
\BR(h^*)&\leq \prob_{x,y}( \inf\limits_{ v \in B_{\eps/c}(\tx)} \tmu^{\sT} \Sigma^{\dagger} v \leq \tfrac{q}{2}\leq  \tx^\sT \Sigma^{\dagger}\tmu, y=+1)\\
&\;+\prob_{x,y} (\tx^\sT \Sigma^{\dagger}\tmu \leq \tfrac{q}{2}\leq \sup\limits_{v \in B_{\eps/c}(\tx)} \tmu^\sT \Sigma^{\dagger}v, y=-1) \,.
\end{align*}
The inner minimization in the above expression can be solved in closed form, by which we obtain 
%
\begin{align*}
\BR(h^*)&\leq \prob_{x,y} \big(\tfrac{q}{2} \leq  \tx^\sT \Sigma^{\dagger}\tmu\leq \tfrac{q}{2}+ \tfrac{\eps}{c}\twonorm{\Sigma^{\dagger}\tmu} , y=1\big)\\
&\;+\prob_{x,y} \big(\tfrac{q}{2}-\tfrac{\eps}{c}\twonorm{\Sigma^{\dagger}\tmu}  \leq  \tx^\sT \Sigma^{\dagger}\tmu \leq \tfrac{q}{2}, y=-1 \big) \,.
\end{align*}

From the Gaussian mixture model \eqref{eq: gaussian-mix} we know that $x\sim \normal(y\tmu,\Sigma)$. For $\tx_+\sim\normal(\tmu,\Sigma)$ and $\tx_-\sim \normal(-\tmu,\Sigma)$ this implies that $\tx|y=+1 \sim \tx_+$, and $\tx|y=-1\sim \tx_-$. By conditioning the above expression on $y$ we get:
\begin{align*}
\BR(h^*)&\leq\pi\; \prob_{\tx_+} \left( 0 \leq \tx_+^\sT \Sigma^{\dagger}\tmu -\tfrac{q}{2} \leq \eps \|\Sigma^{\dagger}\tmu\|/c   \right) \\
&\;+(1-\pi) \; \prob_{\tx_-}\left( - \tfrac{\eps}{c} \|\Sigma^{\dagger}\tmu\|   \leq \tx_-^\sT \Sigma^{\dagger}\tmu -\tfrac{q}{2} \leq 0 \right)\,.
\end{align*}
Since $\tx_+^\sT \Sigma^{\dagger}\tmu \sim \normal(a,a)$, and  $\tx_-^\sT \Sigma^{\dagger}\tmu \sim \normal(-a,a)$ with $a=\tmu^\sT\Sigma^{\dagger}\tmu$, we have 
\begin{align*}
\BR(h^*) &\leq \pi\; \prob_{u\sim \normal(a, a)} \left( 0 \leq u-\tfrac{q}{2} \leq \tfrac{\eps}{c} \|\Sigma^{\dagger}\tmu\|   \right) \\
&\;+ (1-\pi)\; \prob_{u\sim \normal(-a,a) }\left(  - \tfrac{\eps}{c} \|\Sigma^{\dagger}\tmu\| \leq u -\tfrac{q}{2} \leq 0 \right)\,.
\end{align*}
We bound the above probabilities, using the fact that the pdf of normal $\normal(\nu,\sigma)$ is bounded by $1/\sqrt{2\pi}\sigma$, for any values of $\sigma>0$ and $\nu$. This implies that
 \begin{align}
\BR(h^*)\leq  \frac{\eps||\Sigma^{\dagger}\tmu||}{c\sqrt{2\pi}||\Sigma^{\dagger 1/2}\tmu||}\,.\label{eq: tmp1}
\end{align}
In the next step, by using $\|\Sigma^{\dagger 1/2}\|=1/\sigma_{\min}(W)$ we arrive at 
 \begin{align*}
{c\sqrt{2\pi}\BR(h^*)}\leq  \frac{\eps}{\sigma_{\min}(W)}\,.
\end{align*}

This along with $\eps/\sigma_{\min}(W)=o_d(1)$ assumption as $d$ grows to infinity completes the proof.




\subsection{Proof of Proposition \ref{propo: example}}\label{proof: propo: example}
From the definition of boundary risk in \eqref{eq: BR} we have
\[
\BR(h)=\prob\left(h(x)y\geq 0, \inf\limits_{x'\in B_\eps(x)}^{} h(x')h(x)\leq 0\right)\,.
\]
Note that the above probability involves the randomness of $\mu\sim \normal(0,I_k/k)$, and entries of $W$ being $\normal(0,1/k)$. In the first step, for the classifier $h(x)=\sign(x^\sT e_1)$, expand the boundary risk for each possible values of $y\in \{+1,-1\}$ to get 
\begin{align*}
\BR(h)&=\prob\left(x^\sT e_1\geq 0, \inf\limits_{\pnorm{x'-x}\leq \eps_p}^{} \sign(x'^\sT e_1) \leq 0, y=+1\right)\\
&\;+\prob\left(x^\sT e_1\leq 0, \sup\limits_{\pnorm{x'-x}\leq \eps_p}^{} \sign(x'^\sT e_1) \geq 0, y=-1\right)\,.
\end{align*}
In the inner optimization problem, since only the first coordinate of $x'$ is on the scene, it is easy to obtain the optimal adversarial perturbation. This gives us
\begin{align*}
\BR(h)&=\prob\left(0 \leq x^\sT e_1\leq \eps_p, y=+1\right)\\
&\;+\prob\left( -\eps_p \leq x^\sT e_1\leq 0, y=-1  \right)\,.
\end{align*}
Note that in the setting \eqref{eq: gaussian-mix}, $\vph(t)=t$ hence we have $x|y \sim \normal(yW\mu, WW^\sT)$. By conditioning on the $W,\mu$ we get  
\begin{align*}
\BR(h)&=\pi\cdot\E_{W,\mu}\left[\prob_{x\sim \normal(W\mu, WW^\sT) }\left(0 \leq x^\sT e_1\leq \eps_p\right)\right]\\
&\;+(1-\pi)\cdot \E_{W,\mu}\left[\prob_{x\sim \normal(-W\mu, WW^\sT)}\left( -\eps_p \leq x^\sT e_1\leq 0\right)\right]\,.
\end{align*}
We next denote the first column of $W^\sT$ by $\omega\in \reals^k$. It is easy to observe that conditioned on $\mu$ and $\omega$, the linear term $x^\sT e_1$ has a Gaussian distribution with the mean $\mu^\sT \omega$, and the unit variance (recall that $\omega$ lies on the unit $k$ dimensional sphere). This brings us
\begin{align*}
\BR(h)&=\pi\cdot\E_{\mu,\omega}\left[ \prob_{u\sim \normal(\mu^\sT \omega,1)}(0\leq u \leq \eps_p)\right]\\
&\;+(1-\pi)\cdot\E_{\mu,\omega}\left[ \prob_{u\sim \normal(-\mu^\sT \omega,1)}(-\eps_p \leq u \leq 0)\right] \,.
\end{align*}
We then use the standard normal c.d.f.  $\Phi$ to rewrite the above probabilities. This gives us the following:
\begin{align*}
\BR(h)&=\pi\cdot\E_{\mu,\omega}\left[ \Phi(\eps_p-\mu^\sT\omega)-\Phi(-\mu^\sT\omega)\right]\\
&\;+(1-\pi)\cdot\E_{\mu,\omega}\left[ \Phi(\mu^\sT\omega)-\Phi(-\eps_p+\mu^\sT \omega)  \right] \\
&=\E_{\mu,\omega}\left[\Phi\left({\eps_p-\mu^\sT\omega}{}\right) - \Phi\left({-\mu^\sT\omega}{}  \right)   \right]\,,
\end{align*}
where the last equation comes from $~\Phi(-t)=1-\Phi(t)$, for every real value $t$. 
We next use $\mu\sim \normal(0,I_k/k)$ and $\twonorm{\omega}^2=1$. This implies that conditioned on $\omega$,  the random value $\mu^\sT\omega$ has $\normal(0,1/k)$ distribution. In the next step, by using the law of iterated expectations we get 
 \begin{align*}
\BR(h)&=\E_{\mu,\omega}\left[\Phi\left({\eps_p-\mu^\sT\omega}{}\right) - \Phi\left({-\mu^\sT\omega}{}  \right)   \right]\\
&=\E_{\omega}\left[\E\left[\Phi\left({\eps_p-\mu^\sT\omega}{}\right) - \Phi\left({-\mu^\sT\omega}{}  \right)   \right]\big |\omega\right]\\
&=\E_{z\sim \normal(0,1/k)}\left[\Phi\left({\eps_p}{}+z\right) - \Phi(z)  \right] \\
&\geq \E_{z\sim \normal(0,1)}\left[\Phi\left({\eps_p}{}+z\right) - \Phi(z)  \right]=c_{\eps_p}\,,
 \end{align*}
 where the last inequality follows from Lemma \ref{lemma: Phi-decreasing}. It is worth to mention that $c_{\eps_p}$ is a deterministic value, which is independent of the dimensions $k,d$. We next present lemma \ref{lemma: Phi-decreasing} along with its proof. 
\begin{lemma} \label{lemma: Phi-decreasing}
For every nonnegative $\eps$, the following function is nonincreasing in $\sigma>0$:
\[
g(\sigma):=\E_{z\sim \normal(0,1)}\left[ \Phi(\eps+\sigma z) - \Phi(\sigma z)\right]\,.
\]
\end{lemma}
\begin{proof}[Proof of Lemma \ref{lemma: Phi-decreasing}]
Let $\vph(.)$ denote the standard normal pdf. We will show that $g(\sigma)$ has a nonpositive derivative. 
\begin{align*}
\frac{\partial}{\partial \sigma}g(\sigma)&= \E_{z\sim \normal(0,1)}\left[ z\vph(\eps+\sigma z)-z\vph(\sigma z) \right]\\
&= \E_{z\sim \normal(0,1)}\left[ z\vph(\eps+\sigma z)\right]\,,
\end{align*}
where the last equality follows from the fact that $z\vph(\sigma z)$ is an odd function. In the next step, by rewriting the above relation in terms of positive values of $z$, we arrive at

\[
\frac{\partial}{\partial \sigma}g(\sigma)=\frac{1}{\sqrt{2\pi}}\int_{0}^{+\infty}z e^{-\frac{\eps^2+\sigma^2z^2}{2}}\left( e^{-\eps\sigma z}-e^{\eps \sigma z} \right)\vph(z)dz\,.
\]
By noting that for $z\geq0$ we have $e^{-\eps\sigma z}\leq e^{\eps \sigma z}$, we realize that the derivative of $g$ is nonpositive. This completes the proof.
\end{proof}

\subsection{Proof of Theorem \ref{thm: generative_new}}\label{proof: thm: generative-new}
 It is easy to observe that a similar argument in the proof of Theorem \ref{thm: Gaussian-mix} can be adopted here to show that if the result holds for $p=2$, it must hold for $p\geq 2$. It is inspired by the fact that for $p\geq 2$ we have $\twonorm{x'-x}\leq \pnorm{x'-x}d^{\frac{1}{2}-\frac{1}{p}}$. More details can be seen in Section \ref{proof: thm: Gaussian-mix}. We now focus on proving the theorem for $\ell_2$-bounded adversaries of power $\eps$. In the GLM setting, from Corollary \ref{coro: optimals}, the Bayes-optimal classifier is given by
\begin{equation}\label{eq: tmp-generative-1}
h^*(x)=\sign\left(f(\beta^\sT \tW^\sT \tx)-1/2\right)\,,
\end{equation}
where $\tW=W(W^\sT W)^{-1}$, and $\tx=\vph^{-1}(x)$.
We focus on computing the boundary risk of $h^*$.  In the first step, from the boundary risk definition we have
\begin{align*}
\BR(h^*)=\prob\left(h^*(x)y\geq 0, \inf\limits_{u\in B_ \eps(x)}^{} h^*(x)h^*(u)\leq 0\right) \,.
\end{align*}
By conditioning on the value of $y$ we get
\begin{align*}
\BR(h^*)&= \prob\left(y=+1, h^*(x)\geq 0, \inf\limits_{u\in B_\eps(x)}^{} h^*(u)\leq 0\right)\\
&\;+\prob\left(y=-1, h^*(x)\leq 0, \sup\limits_{u\in B_\eps(x)}^{} h^*(u)\geq 0\right)\,.
\end{align*}
 By removing conditions $y=+1$ and $y=-1$, we can upper bound the above probabilities. This gives us 
\begin{align*}
\BR(h^*)&\leq \prob\left(h^*(x)\geq 0, \inf\limits_{u\in B_\eps(x)}^{} h^*(u)\leq 0\right)\\
&\;+\prob\left(h^*(x)\leq 0, \sup\limits_{u\in B_\eps(x)}^{} h^*(u)\geq 0\right)\,.
\end{align*}
In the next step, using \eqref{eq: tmp-generative-1} yields

\begin{align}\label{eq: tmp-generative-1}
\BR(h^*)&\leq\prob\left(\inf\limits_{u\in B_\eps(x)}^{} f\left(\vph^{-1}(u)^{\sT}\tW\beta\right) \leq 1/2 \leq f( \tx^\sT\tW\beta  ) \right) \nonumber\\
&\;+\prob\left(f(\tx^\sT\tW\beta  )\leq 1/2\leq \sup\limits_{u\in B_\eps(x)}^{} f\left(\vph^{-1}(u)^\sT\tW\beta\right) \right)\,.
\end{align}

From the manifold models in Section \ref{sec:data-model}, we know that there exists $c>0$ such that the derivative of $\vph$ satisfies $d\vph/dt\geq c$. This implies that $\vph^{-1}\left(B_\eps(x)\right)$ is a subset of $B_{\eps/c}\left(\vph^{-1}(x)\right)$. Using this in \eqref{eq: tmp-generative-1} brings us 
\begin{align}\label{eq: tmp-generative-2}
\BR(h^*) &\leq \prob\left(\inf\limits_{v\in B_{\eps/c}(\tx) }^{} f\left(v^{\sT}\tW\beta\right) \leq 1/2 \leq f( \tx^\sT\tW\beta) \right) \nonumber\\
&\;+ \prob\left(f(\tx^\sT\tW\beta  )\leq 1/2 \leq \sup\limits_{v\in B_{\eps/c}(\tx) }^{} f\left(v^\sT\tW\beta\right)   \right)\,.
\end{align}

 Function $f$ is an increasing function, therefore the inner optimization problems can be cast with linear objectives and bounded $\ell_2$-ball constraints. This means that it is feasible to characterize a closed form solution for the optimization problem. For this end, we first introduce $c_0=f^{-1}(1/2)$ and rewrite \eqref{eq: tmp-generative-2} as
\begin{align}\label{eq: tmp12-1}
\BR(h^*)&\leq \prob\left(c_0 \leq  \tx^\sT\tW\beta \leq c_0+ \eps\twonorm{\tW\beta}/c  \right)  \nonumber\\
&\;+\prob\left(c_0-\eps\twonorm{\tW\beta}/c \leq  \tx^\sT\tW\beta  \leq c_0  \right).
\end{align}
As $\tx=Wz$, therefore $\tx\sim \normal(0,WW^\sT)$. This implies that $\tx^\sT\tW\beta \sim \normal(0,\twonorm{\beta}^2 )$, and the probabilities in \eqref{eq: tmp12-1} can be computed by standard normal cdf $\Phi$. By Letting $\gamma=\tW\beta$, we arrive at
\begin{align*}
\BR(h^*)\leq  \Phi\left(\frac{c_0+\eps\twonorm{\gamma}/c}{\twonorm{\beta}} \right) - \Phi \left(\frac{c_0-\eps\twonorm{\gamma}/c}{\twonorm{\beta}}\right) \,.
\end{align*}
By using the fact that $\Phi$ is $1/\sqrt{2\pi}$-Lipschitz continuous, we get 
\begin{align}\label{eq: tmp12-2}
\BR(h^*)&\leq \frac{2\eps\twonorm{\tW\beta}}{c\sqrt{2\pi}\twonorm{\beta}}\,.
\end{align}
In the next step, from $\tW^\sT \tW=(W^\sT W)^{-1}$ we realize that $\twonorm{\tW\beta}\leq \twonorm{\beta}/\sigma_{\min}(W)$. Using this in \eqref{eq: tmp12-2} yields

\begin{align*}
\BR(h^*)&\leq \frac{2\eps}{c\sqrt{2\pi}\sigma_{\min}(W)}\,.
\end{align*}
Deploying $\eps/\sigma_{\min}(W)=o_d(1)$ completes the proof.

\subsection{Proof of Theorem \ref{thm: agnostic}}\label{proof:thm: agnostic}
 Similar to the proof of theorems in the previous sections,  an analogous argument can be used here to show that if the result holds for $p=2$, it must hold for $p\geq 2$.  In  short, it comes from the fact that  for $p\geq 2$ we have $\twonorm{x'-x}\leq \pnorm{x'-x}d^{\frac{1}{2}-\frac{1}{p}}$. More details can be seen in Section \ref{proof: thm: Gaussian-mix}. We now focus on proving the theorem for $\ell_2$-bounded adversaries of size $\eps$.
By recalling the definition of the boundary risk we get
\[
\BR(h_\th)=\prob(h_\th(x)y\geq 0, \inf\limits_{x'\in B_{\eps}(x)}^{} h_\th(x')h_\th(x)\leq 0)\,.
\]
We expand the above probabilities with respect to the $y$ possible values $\{+1,-1\}$. This gives us
\begin{align*}
\BR(h_\th)&=\prob(y=+1,h_\th(x)\geq 0, \inf\limits_{x'\in B_{\eps}(x)}^{} h_\th(x')\leq 0)\\
&\;+\prob(y=-1, h_{\th}(x)\leq 0, \sup\limits_{x'\in B_{\eps}(x)}^{} h_\th(x')\geq 0 )\,.
\end{align*}
 Plugging $h_\th(x)=\sign(x^\sT\th)$ into to the above expression yields
\begin{align*}
\BR(h_\th)&=\prob(y=+1,x^\sT\th\geq 0, \inf\limits_{x'\in B_{\eps}(x)}^{} \th^\sT x'\leq 0)\\
&\;+\prob(y=-1, x^\sT \th \leq 0, \sup\limits_{x'\in B_{\eps}(x)}^{} \th^\sT x'\geq 0 )\,.
\end{align*}
By solving the inner optimizations we can get the following
 \begin{align*}
\BR(h_\th)&=\prob(y=+1, 0\leq x^\sT \th\leq \eps\twonorm{\th})\\
&\;+\prob(y=-1, -\eps\twonorm{\th}\leq x^\sT\th\leq 0  )\,.
\end{align*} 
 From the Gaussian mixture model \eqref{eq: gaussian-mix} we know that $x|y\sim \normal(yW\mu,WW^\sT)$.  By using this Gaussian distribution along with conditioning on values of $y$, we get
\begin{align*}
\BR(h_\th)&=\pi\cdot \prob_{x\sim \normal(W\mu, WW^\sT)}(0\leq x^\sT \th\leq \eps\twonorm{\th})\\
&\;+(1-\pi)\cdot \prob_{x\sim \normal(-W\mu, WW^\sT)}( -\eps\twonorm{\th}\leq x^\sT \th\leq 0 )\,.
\end{align*}
As $x$ has multivariate Gaussian distribution, we can rewrite the above probabilities in terms of the standard normal cdf $\Phi$. In addition, by using $\Phi(-t)=1-\Phi(-t)$, we get
\[
\BR(h_\th)= \Phi\left( \frac{\eps\twonorm{\th}-\th^\sT W\mu}{\twonorm{W^\sT\th}} \right) -\Phi\left( \frac{-\th^\sT W\mu}{\twonorm{W^\sT\th}} \right)\,. 
\]

In the next step, by using the fact that the normal cdf $\Phi$ is $1/\sqrt{2\pi}$ Lipschitz continuous we arrive at
 \begin{equation}\label{eq: agnostic-tmp1}
 \BR(h_\th)\leq \frac{\eps\twonorm{\th} }{\sqrt{2\pi}\twonorm{W^\sT\th}}\,.
 \end{equation}

Projection of the decision parameter $\th$ onto the kernel of the weight matrix $W$ gives the decomposition $\th=\th_0+W\alpha$,  for $\alpha\in \reals^k, \th_0\in \reals^d$ with $\th_0=P_{\mathsf{ker}(W^\sT)}(\th)$. Using $\twonorm{\th}^2=\twonorm{\th_0}^2+\twonorm{W\alpha}^2$ in \eqref{eq: agnostic-tmp1} yields  

\begin{align*}
 \BR(h_\th)\leq \frac{\eps\twonorm{W\alpha} }{\sqrt{2\pi}\twonorm{W^\sT\th}}{\left({1-\frac{\twonorm{\th_0}^2}{\twonorm{\th}^2}}\right)^{-1/2}} \,.
\end{align*}

We then use $W^\sT\th=W^\sT W\alpha$ and the operator norm property $\twonorm{W\alpha}\sigma_{\min}(W)\leq \twonorm{W^\sT W\alpha}$. This brings us the following

\begin{align*}
 \BR(h_\th)\leq \frac{\eps}{\sqrt{2\pi}\sigma_{\min}(W)} {\left({1-\frac{\twonorm{\th_0}^2}{\twonorm{\th}^2}}\right)^{-1/2}}\,.
\end{align*}

Finally, deploying the problem assumption \eqref{eq: agnostic-conditons} completes the proof. 

In this part, we focus on the ERM problem. First, note that as the loss function $\ell$ is decreasing, it is easy to observe that the supremum of $yu^\sT\th$ over the adversarial ball $u \in B_\eps(x_i)$ is $\ell(yx^\sT\th-\eps\twonorm{\th})$. This implies that the adversarial ERM problem can be written as the following:
\begin{equation}\label{eq: agnostic-ERM}
 \arg\min\limits_{\th\in {\reals^d}}^{} R_n(\th):=\frac{1}{n}\sum\limits_{i=1}^{n} \ell(y_ix_i^\sT\th-\eps\twonorm{\th})\,.
\end{equation}
In order to show that the classifier $h_{\hth^\eps}$ has boundary risk converging to zero, we use the first part of the theorem. For this purpose, we will show that indeed the obtained solution $\hth^\eps$ has no component in the kernel space of the matrix $W$. In this case, the general condition \eqref{eq: agnostic-conditons} will reduce to $\frac{\eps_pd^{\frac{1}{2}-\frac{1}{p}}}{\sigma_{\min}(W)}=o_d(1)$. This means that we only need to prove that $P_{\mathsf{Ker}(W^\sT)}(\hth^\eps)=0$. For this end,  by projecting $\hth^\eps$ onto the kernel of $W$ we get
$\hth^\eps=\th_0+\th_1$ with $W^\sT\th_0=0$, and there exists $\alpha\in \reals^k$ such that $\th_1=W\alpha$.
We want to show that $\th_0=0$. Assume that $\twonorm{\th_0}> 0$, we show this will contradict the fact that $\hth^\eps$ is a minimizer of $R_n(\th)$.  By plugging $\hth^\eps=\th_0+\th_1$ into \eqref{eq: agnostic-ERM} we get
\begin{align*}
R_n(\hth^\eps) &= \frac{1}{n}\sum\limits_{i=1}^{n} \ell(y_ix_i^\sT\hth^\eps-\eps\|\hth^\eps\|_2)\\
&=\frac{1}{n}\sum\limits_{i=1}^{n} \ell(y_ix_i^\sT\th_1-\eps\|\hth^\eps\|_2)\,,
\end{align*}
where the last equality follows by the fact that $x_i=Wz_i$ with $z_i\sim \normal(y\mu_i,I_k)$, and $W^\sT\th_0=0$. In addition, from the orthogonality of $\th_0$ and $\th_1$ we get
\begin{align*}
R_n(\hth^\eps)= \frac{1}{n}\sum\limits_{i=1}^{n} \ell\left(y_ix_i^\sT\th_1-\eps \left( \twonorm{\th_0}^2+\twonorm{\th_1}^2\right)^{1/2}  \right)\,.
\end{align*}
Finally, since the loss function $\ell$ is strictly decreasing, therefore for $\twonorm{\th_0}>0$ we get 
\begin{align*}
R_n(\hth^\eps)&> \frac{1}{n}\sum\limits_{i=1}^{n} \ell\left(y_ix_i^\sT\th_1-\eps \twonorm{\th_1}  \right)\\
&=R_n(\th_1)\,.
\end{align*}
This contradicts the initial assumption that $\hth^\eps$ is a minimizer of $R_n(\th)$. This means that $\th_0=0$, and completes the proof.

\if false
\section{Additional numerical experiments}
In this section, we provide additional numerical experiments to corroborate our claim that the boundary risk converges to zero when the dimensions ratio $d/k$ grows to infinity. For this purpose, similar to the setting of the numerical experiment presented in Figure \ref{fig:main}, we focus on the class of  linear classifiers $\sign(x^\sT \th)$ with model $\th$ coming from the robust ERM problem \eqref{eq: tmp9}. 
We construct a dataset of $n=300$ samples, generated from the Gaussian mixture model \eqref{eq: gaussian-mix} with balanced classes ($p=1/2$). In addition, we assume that the weight matrix $W$ and the mean vector $\mu$ have i.i.d. entries $\normal(0, 1/k)$.  We fix the ambient dimension at $d=100$, and vary the manifold dimension $k$ from $1$ to $100$. We consider four different feature mappings: $(i)$ $\vph_1(t)=t$, $(ii)$ $\vph_2(t)=t/4+\sign(t) 3t/4$, $(iii)$ $\vph_3(t)=t+\sign(t)t^2$, and $(iv)$ $\vph_4(t)=\tanh(t)$. The plots in Figure \ref{fig:app-1} showcase the behavior of the standard, adversarial, and the boundary risks for each of these mappings and for the adversary's power $\eps = 1$. In Figure~\ref{fig:app-2}, we depict the boundary risk for different choices of $\eps$.  As observed, the boundary risk decreases to zero, as the dimensions ratio $d/k$ grows to infinity.

\begin{figure*}
	\centering
	\begin{subfigure}[b]{0.4\textwidth}
		\centering
		\includegraphics[scale=0.4]{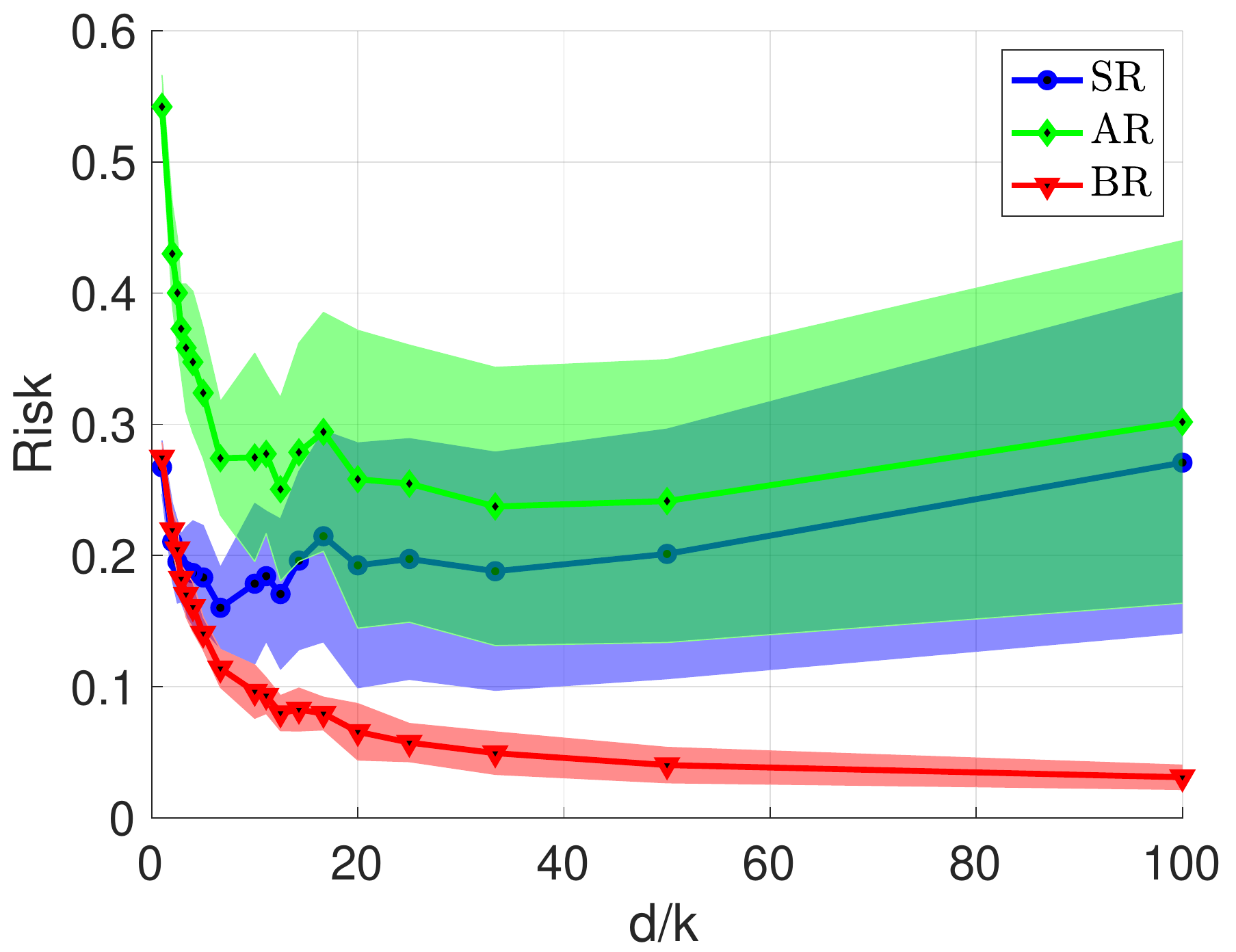}
		\caption{Feature mapping $\vph(t)=t$ and adversary's power $\eps=1$ }
		\label{fig:app:identity-1}
	\end{subfigure}%
	\hfill
	\begin{subfigure}[b]{0.4\textwidth}
		\centering
		\includegraphics[scale=0.4]{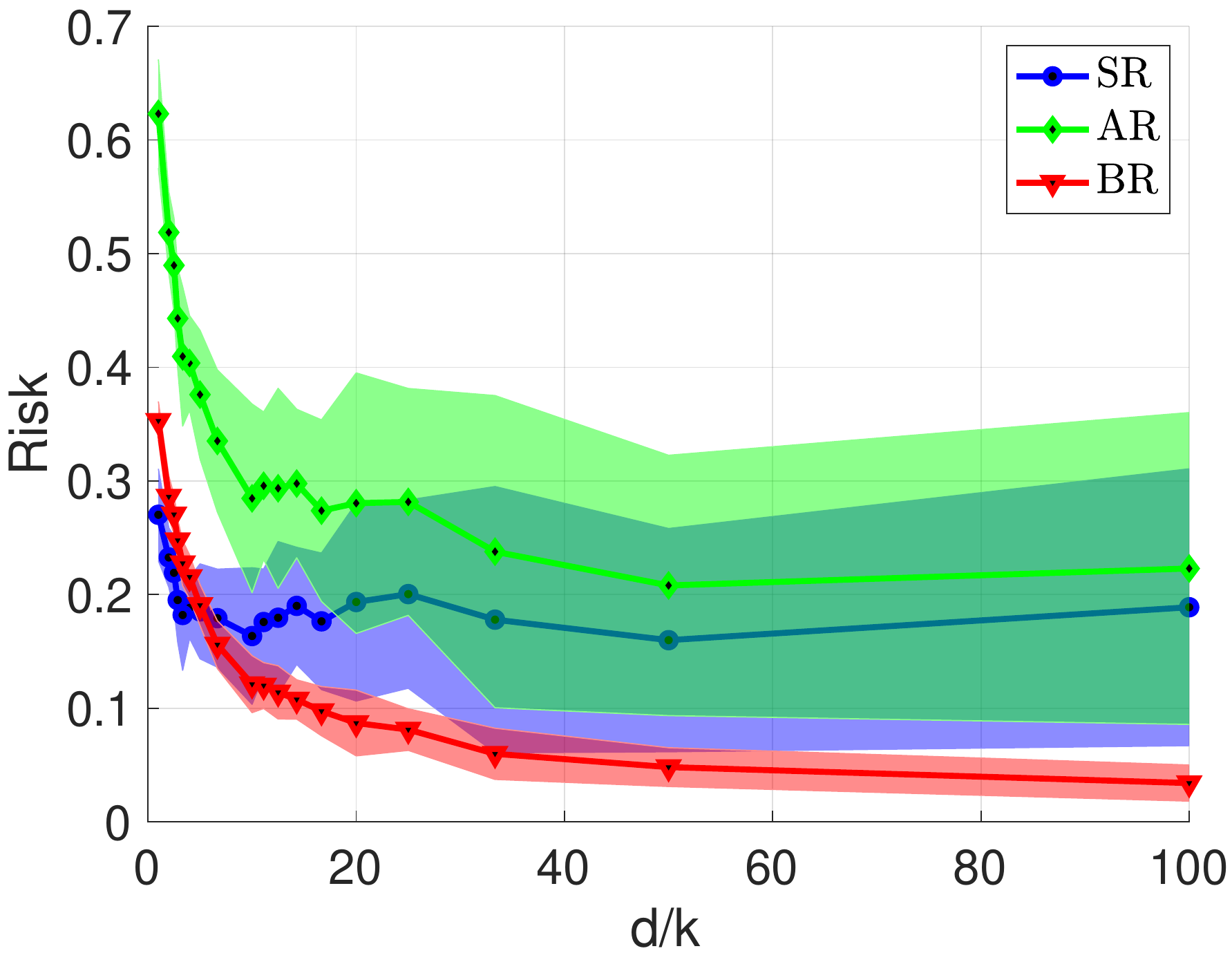}
		\caption{Feature mapping $\vph(t)=t/4+\sign(t)3t/4$ and adversary's power $\eps=1$}
		\label{fig:app:leaky-1}
	\end{subfigure}
		
	\begin{subfigure}[b]{0.4\textwidth}
		\centering
		\includegraphics[scale=0.54]{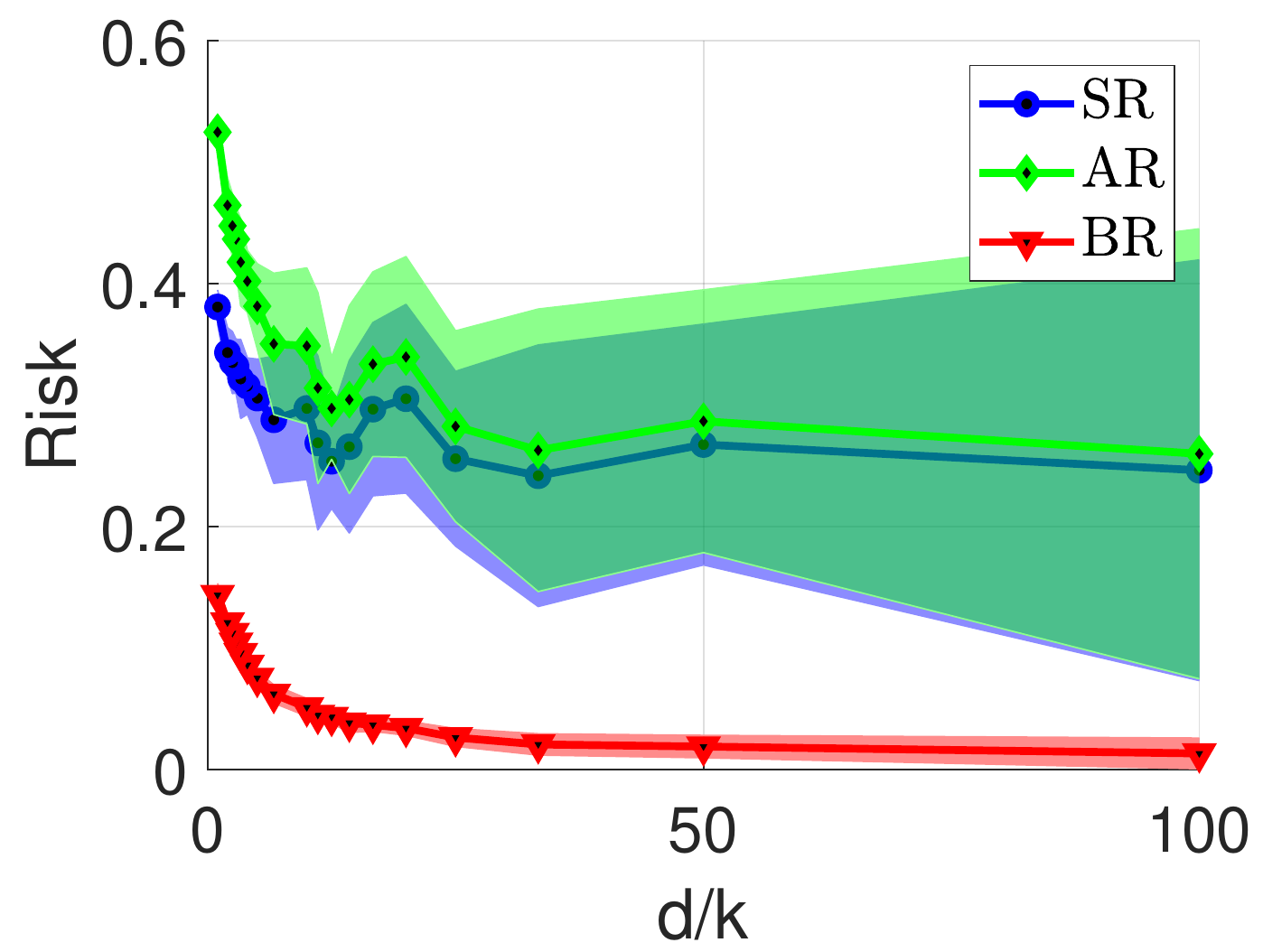}
		\caption{Feature mapping $\vph(t)=t+\sign(t)t^2$ and adversary's power $\eps=1$}
		\label{fig:app:quadratic-1}
	\end{subfigure}%
	\hfill
	\begin{subfigure}[b]{0.4\textwidth}
		\centering
		\includegraphics[scale=0.54]{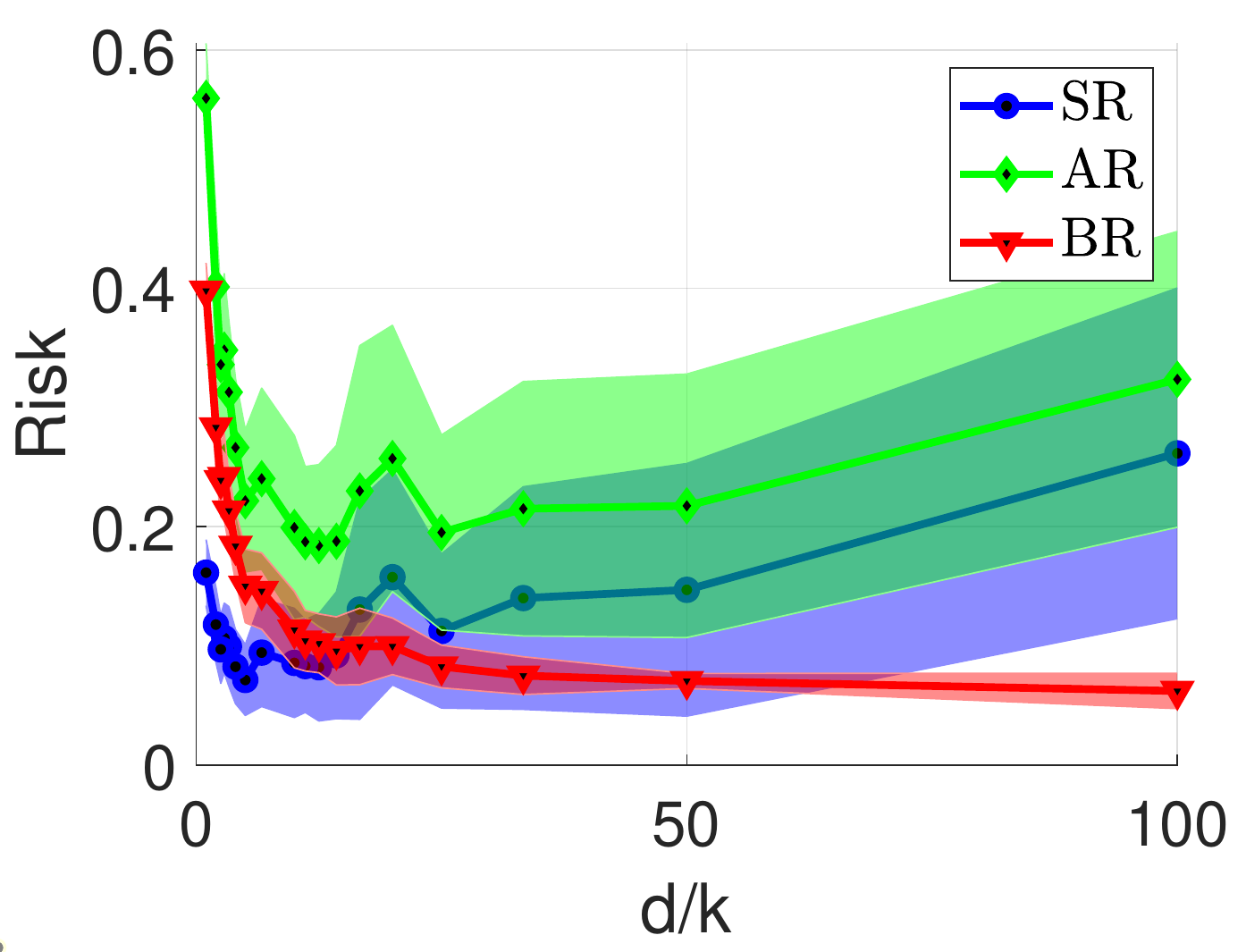}
		\caption{Feature mapping $\vph(t)=\tanh(t)$ and adversary's power $\eps=1$.}
		\label{fig:app:tanh-1}
	\end{subfigure}
	\caption{Effect of dimensions ratio $d/k$ on the standard, adversarial, and boundary risks of the linear classifier $h_\th(x)=\sign(x^\sT \th)$ with $\th$ being the robust empirical risk minimizer~\eqref{eq: tmp9}. 
Samples are generated from the Gaussian mixture model \eqref{eq: gaussian-mix} with balances classes ($p=1/2$), and with four choices of feature mapping $\vph$: (a)\; $\vph(t)=t$, (b)\ $\vph(t)=3t/4+\sign(t) t/4$,
(c)\ $\vph(t)=t+\sign(t)t^2$ and (d)\, $\vph(t)=\tanh(t)$. In these experiments, the ambient dimension $d$ is fixed at $100$, and the manifold dimension $k$ varies from $1$ to $100$. The sample size is $n=300$ the classes average $\mu$ and the weight matrix $W$ have i.i.d. entries from $\normal(0,1/k)$. The adversary's power is fixed at $\eps=1$. For each fixed values of $k$ and $d$, we consider $M=20$ trails of the setup.  Solid curves represent the average results across these trials, and the shaded areas represent one standard deviation above and below the corresponding curves.}
	\label{fig:app-1}
\end{figure*}

\begin{figure*}
	\centering
	\begin{subfigure}[b]{0.47\textwidth}
		\centering
		\includegraphics[scale=0.4]{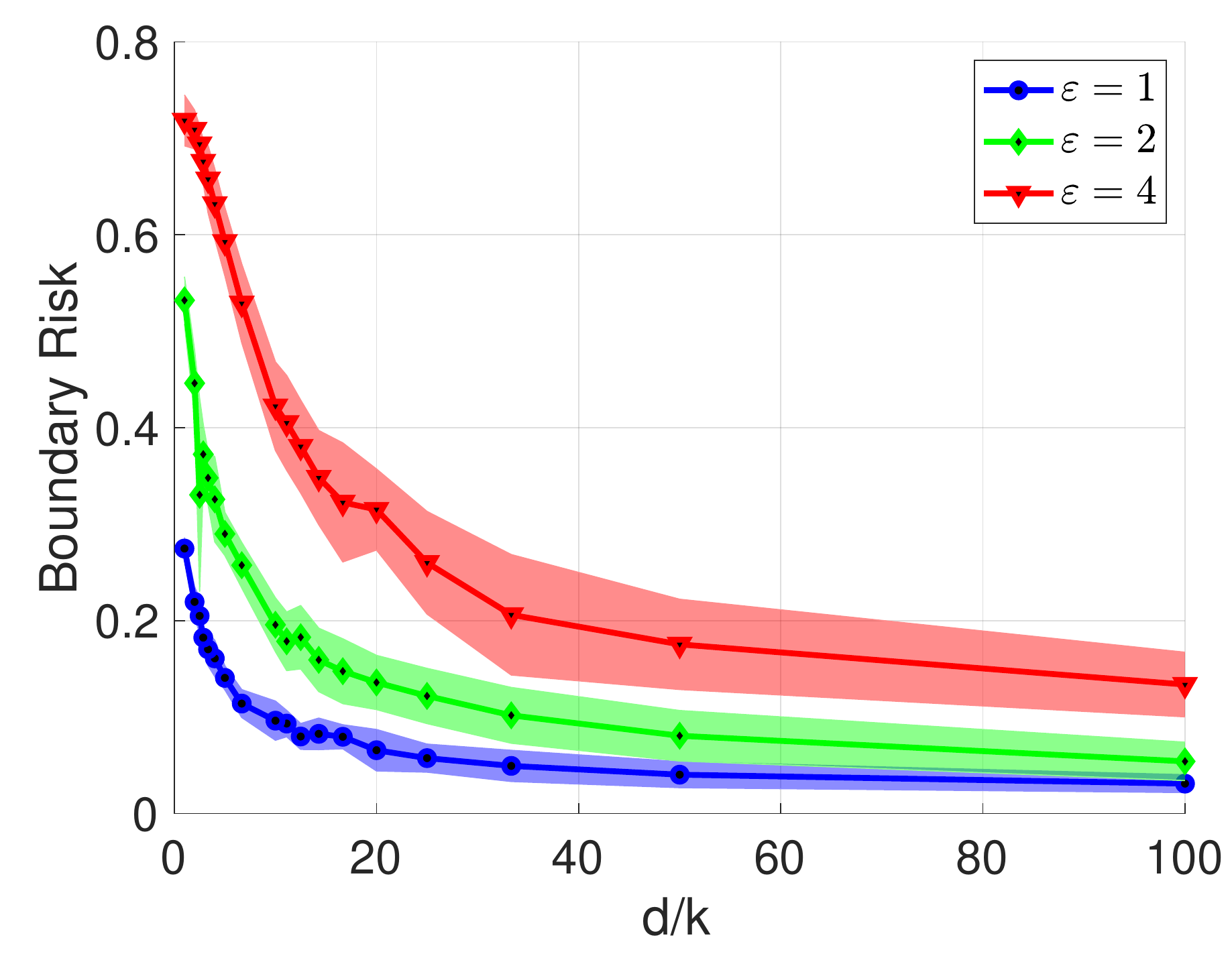}
		\caption{Boundary risk with the feature mapping $\vph(t)=t$ and for multiple values of adversary's power $\eps$ }
		\label{fig:app:identity-2}
	\end{subfigure}%
	\hfill
	\begin{subfigure}[b]{0.47\textwidth}
		\centering
		\includegraphics[scale=0.4]{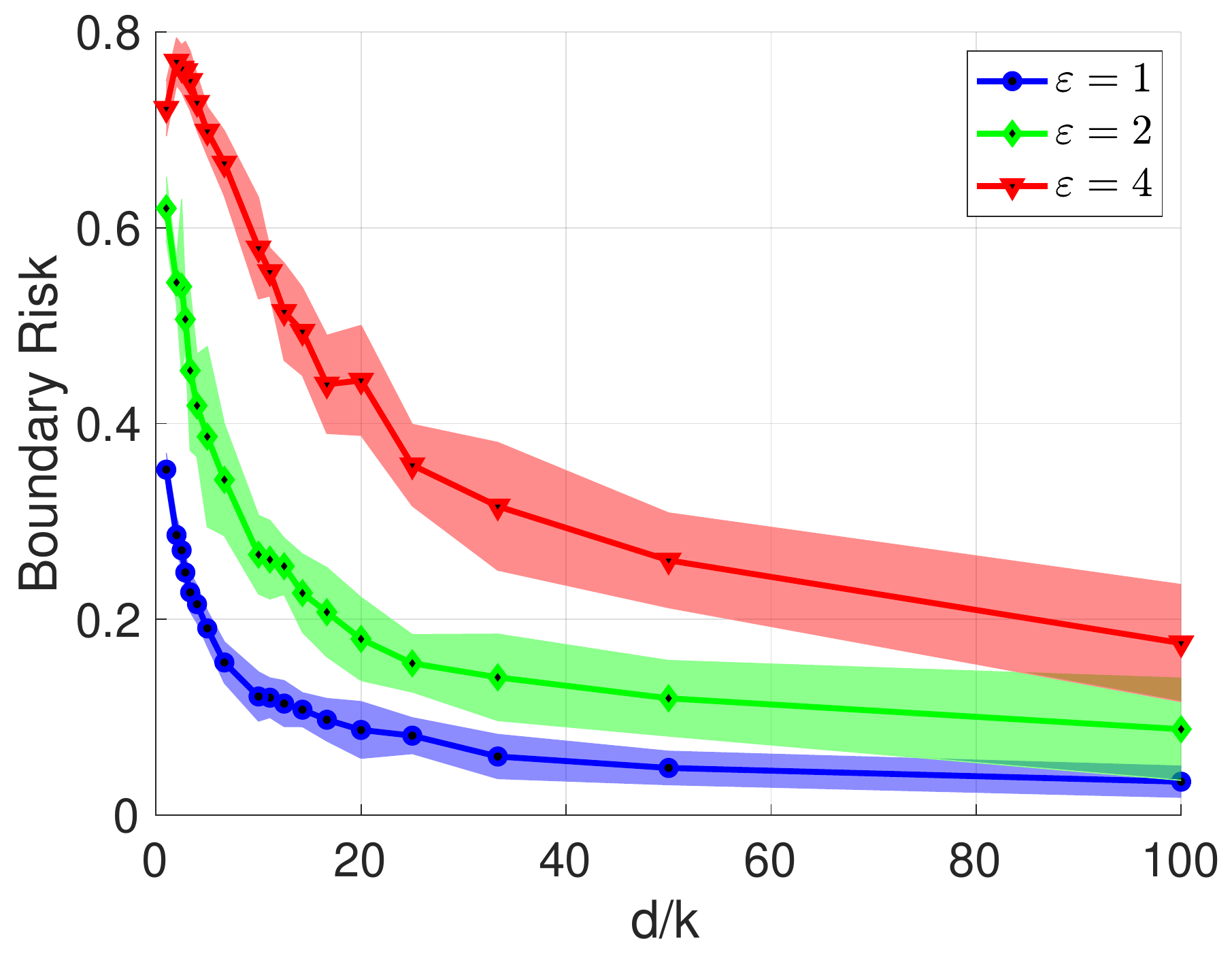}
		\caption{Boundary risk with the feature mapping $\vph(t)=t/4+\sign(t)3t/4$ and for multiple values of adversary's power $\eps$}
		\label{fig:app:leaky-2}
	\end{subfigure}
		
	\begin{subfigure}[b]{0.47\textwidth}
		\centering
		\includegraphics[scale=0.4]{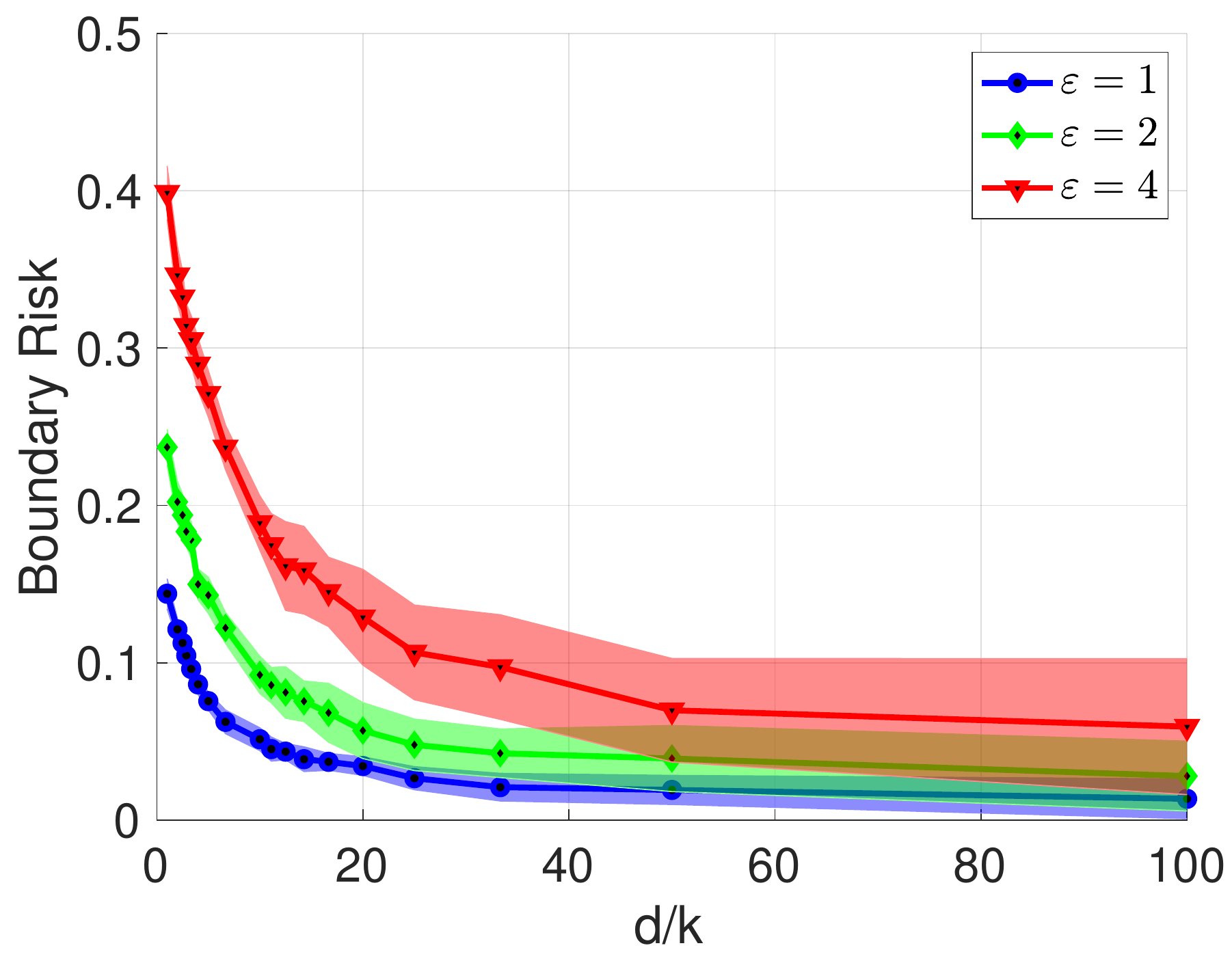}
		\caption{Boundary risk with the feature mapping $\vph(t)=t+\sign(t)t^2$ and for multiple values of adversary's power $\eps$}
		\label{fig:app:quadratic-2}
	\end{subfigure}%
	\hfill
	\begin{subfigure}[b]{0.47\textwidth}
		\centering
		\includegraphics[scale=0.4]{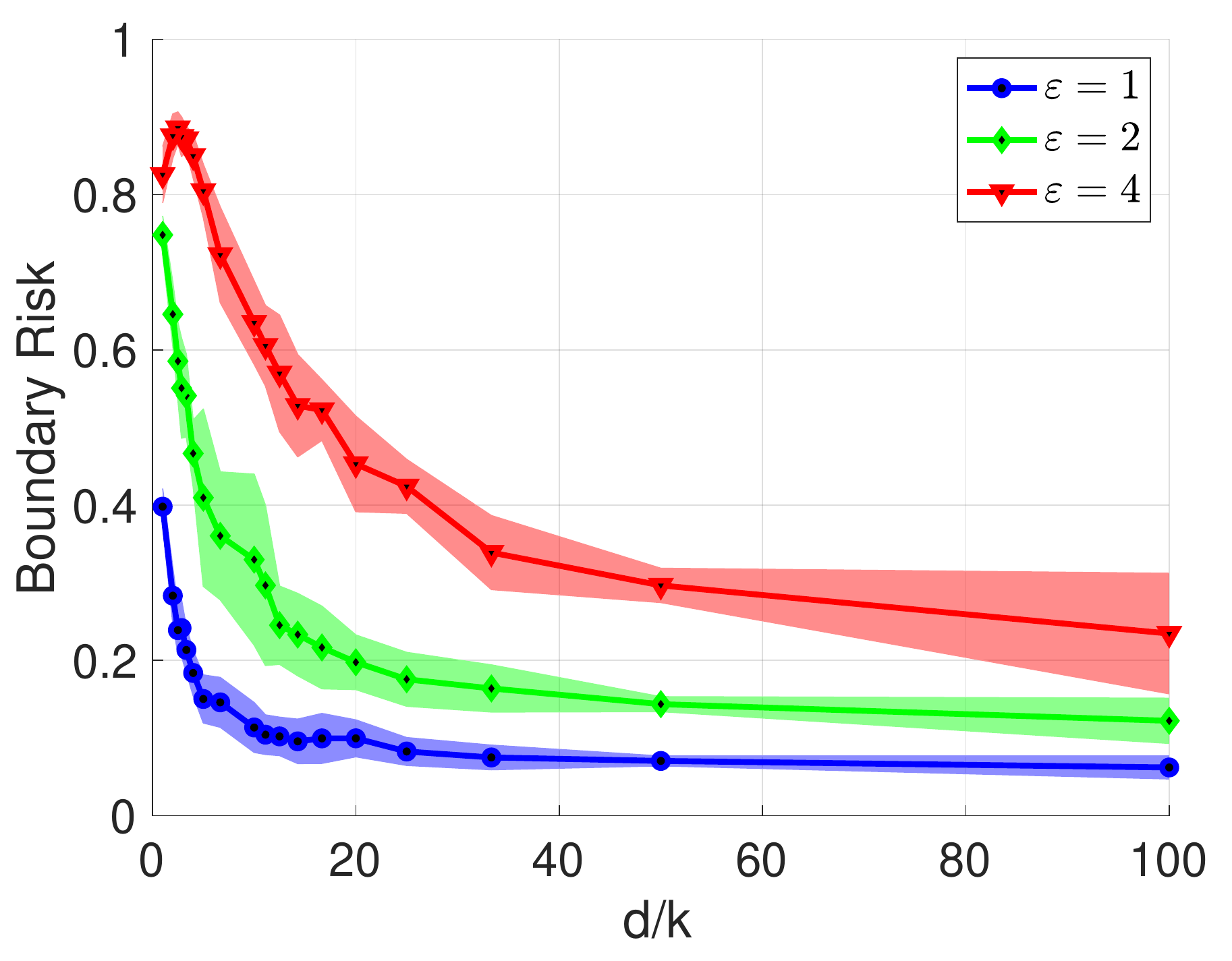}
		\caption{Boundary risk with the feature mapping $\vph(t)=\tanh(t)$ and for multiple values of adversary's power $\eps$}
		\label{fig:app:tanh-2}
	\end{subfigure}
	\caption{Effect of dimensions ratio $d/k$ on the boundary risk of the linear classifier $h_\th(x)=\sign(x^\sT \th)$ with $\th$ being the robust empirical risk minimizer~\eqref{eq: tmp9}. 
Samples are generated from the Gaussian mixture model \eqref{eq: gaussian-mix} with balances classes ($p=1/2$), and with four choices of feature mapping $\vph$: (a)\; $\vph(t)=t$, (b)\ $\vph(t)=3t/4+\sign(t) t/4$,
(c)\ $\vph(t)=t+\sign(t)t^2$ and (d)\, $\vph(t)=\tanh(t)$. In these experiments, the ambient dimension $d$ is fixed at $100$, and the manifold dimension $k$ varies from $1$ to $100$. The sample size is $n=300$ the classes average $\mu$ and the weight matrix $W$ have i.i.d. entries from $\normal(0,1/k)$. We consider different levels of the adversary's power  $\eps\in\{1,2,4\}$. For each fixed values of $k$ and $d$, we consider $M=20$ trails of the setup.  Solid curves represent the average results across these trials, and the shaded areas represent one standard deviation above and below the corresponding curves.}
	\label{fig:app-2}
\end{figure*}

\fi
